\pdfoutput=1

\documentclass[11pt]{article}

\usepackage[preprint]{acl}

\usepackage{times}
\usepackage{latexsym}

\usepackage[T1]{fontenc}

\usepackage[utf8]{inputenc}

\usepackage{microtype}

\usepackage{inconsolata}

\input{definitions/paper_commands}
\title{Federated Data-Efficient Instruction Tuning for Large Language Models}

\author{
 \textbf{Zhen Qin\textsuperscript{1}},~~
 \textbf{Zhaomin Wu\textsuperscript{2}\footnotemark[1]},~~
 \textbf{Bingsheng He\textsuperscript{2}},~~
 \textbf{Shuiguang Deng\textsuperscript{1}\thanks{Corresponding authors.}
}
\\
\textsuperscript{1}Zhejiang University, 
\textsuperscript{2}National University of Singapore
\\
\texttt{zhenqin@zju.edu.cn}, \texttt{zhaomin@nus.edu.sg}, \texttt{hebs@comp.nus.edu.sg}, \texttt{dengsg@zju.edu.cn}\\
}

\begin{document}
\maketitle

\doparttoc
\faketableofcontents

\begin{abstract}
Instruction tuning is a crucial step in improving the responsiveness of pretrained large language models (LLMs) to human instructions.
Federated learning (FL) helps to exploit the use of vast private instruction data from clients, becoming popular for LLM tuning by improving data diversity.
Existing federated tuning simply consumes all local data, causing excessive computational overhead and overfitting to local data, while centralized data-efficient solutions are not suitable for FL due to privacy concerns. 
This work presents \app, a federated data-efficient instruction tuning approach, which tunes LLMs with a representative subset of edge-side data.
It reduces the data redundancy at both intra- and inter-client levels without sharing raw data.
Experiments with various LLMs, datasets and partitions show that \app improves Rouge-L on unseen tasks by an average of 10.72\% over the SOTA full-data federated instruction tuning methods, while using less than 1.5\% of the data samples, improving training efficiency by up to tens of times.
\end{abstract}

\section{Introduction}
\label{sec-intro}
Large language models (LLMs) exhibit remarkable performance on a wide range of natural language tasks. 
Instruction tuning \cite{supernaturalinstructions} is crucial to improve LLMs' responsiveness to human instructions, whose success hinges on the availability of diverse and high-quality data \cite{wang2024survey,wang2023self,li2024selective}, particularly for unseen tasks \cite{wei2022flan}.
As the pool of publicly available data is expected to be exhausted in the near future \cite{villalobos2022will}, exploiting more data sources, such as data on edge devices, becomes essential for continued LLM improvement \cite{qin2024full}.
However, privacy concerns and regulations such as GDPR \cite{voigt2017eu} complicate the use of edge-side data.

Federated learning (FL) \cite{mcmahan2017fl} has emerged as a promising solution to utilize diverse data from edge devices to tune LLMs \cite{zhang2024fedit}, especially for generalization to unseen tasks \cite{qin2024full,bai2024federated}.
Prior FL works on LLMs mainly focus on communication and memory costs \cite{zhang2024fedit,babakniya2023slora,zhang2023-FedPETuning,che2023FedPepTAO,kuang2024federatedscope,xu2024forward,qin2024full}.
Although these methods show promising results, our investigation reveals that they often use all local data for training \cite{qin2024full,xu2024forward,zhang2024fedit,ling2024convergence}. 
This strategy leads to many training steps, and thus causes two main issues in FL contexts:

\begin{figure}[t]
  \centering
  \begin{minipage}[t]{0.48\linewidth}
    \centering
    \includegraphics[width=\linewidth]{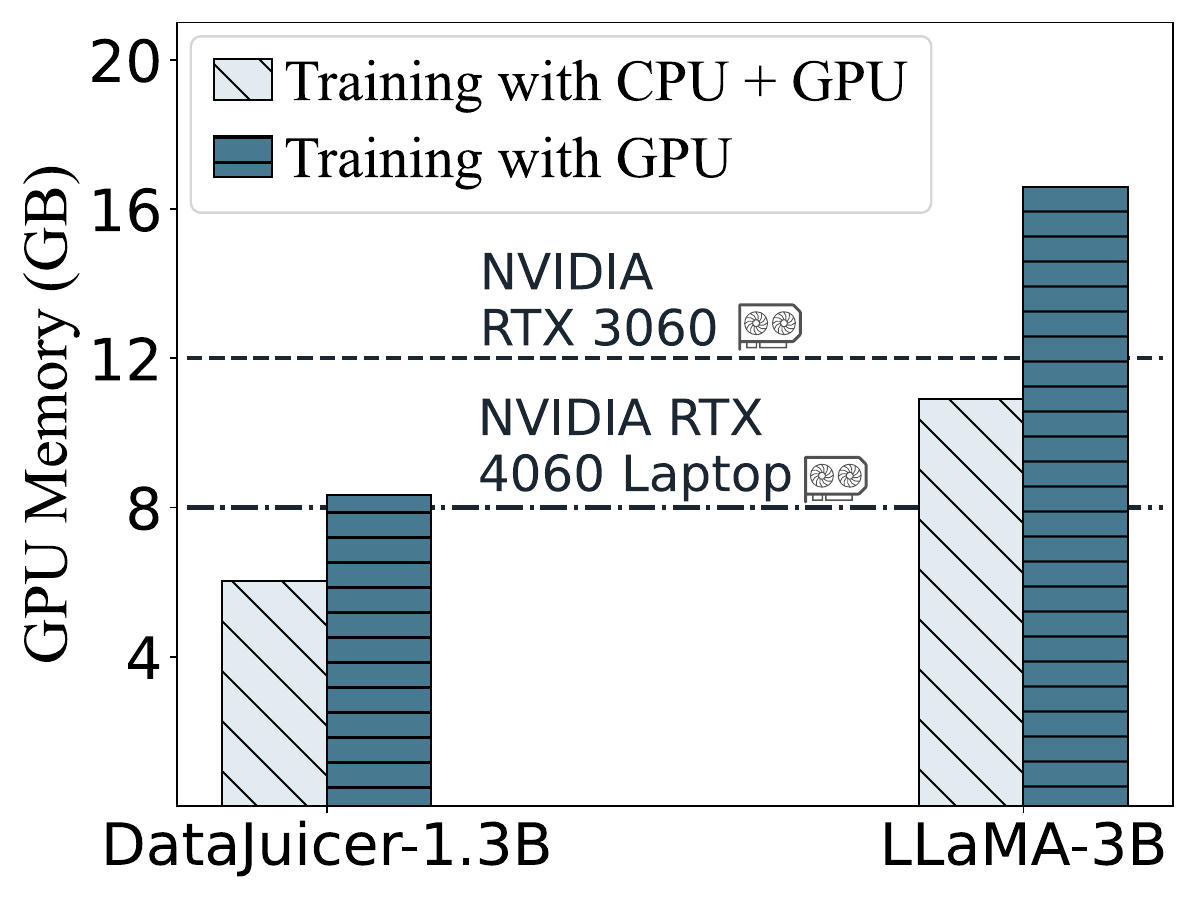}
    \caption{GPU memory cost by training with LoRA\protect\hyperref[fn1]{$^1$} v.s. capacity of popular edge-side GPUs\protect\hyperref[fn2]{$^2$}.}
    \label{fig:intro:mem}
  \end{minipage}
  \hspace{0.1cm}
  \begin{minipage}[t]{0.48\linewidth}
    \centering
    \includegraphics[width=\linewidth]{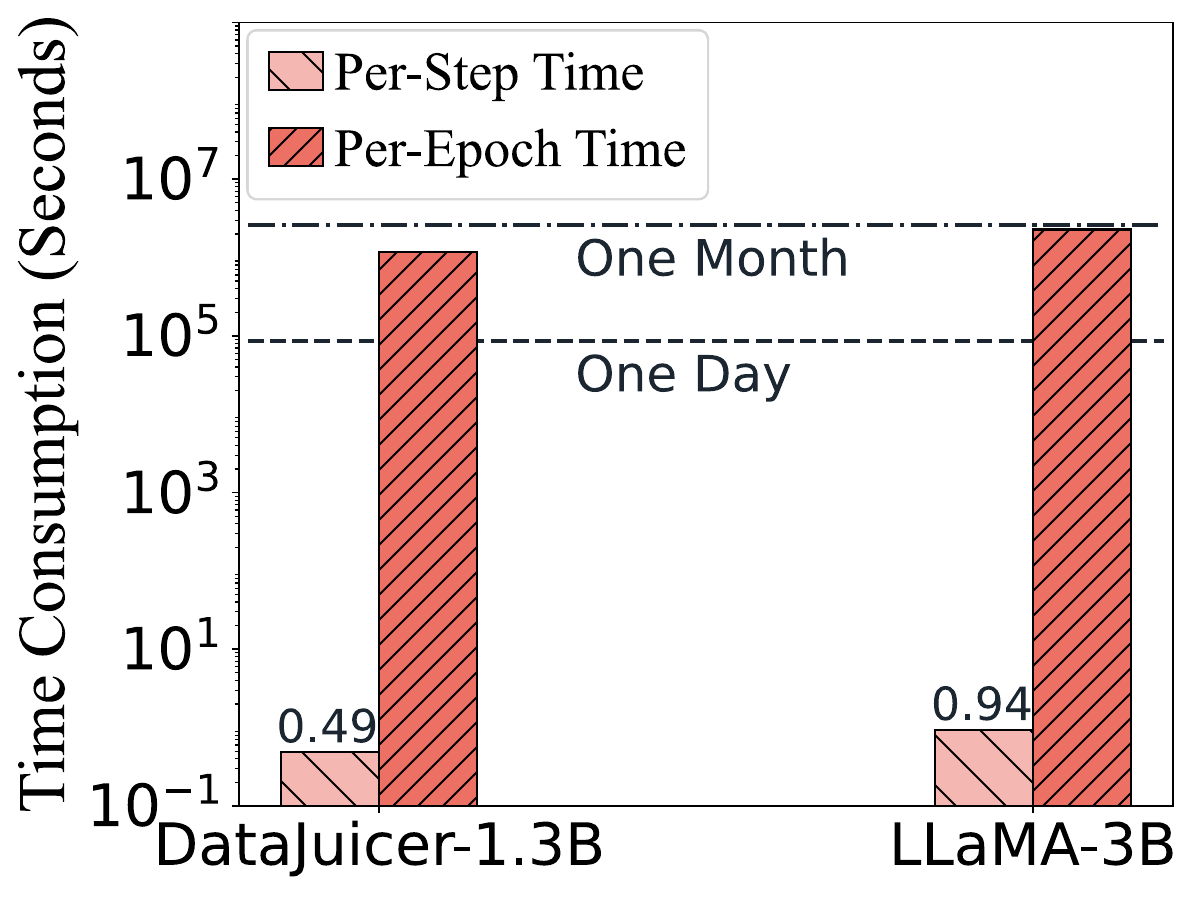}
    \caption{Per-step/epoch time of BP with LoRA on Natural Instructions dataset by CPU+GPU.
    }
    \label{fig:intro:time}
  \end{minipage}
\end{figure}
\footnotetext[1]{\label{fn1}~CPU+GPU training is implemented with Deepspeed \cite{rasley2020deepspeed} and a batch size of 1 (loaded in 16-bit).}
\footnotetext[2]{\label{fn2}~The most popular desktop and laptop GPUs are reported by Steam Hardware Survey (Dec 2024).}
(1) Efficiency: 
It is time-consuming to train LLMs, especially when edge devices have limited GPU capacity and thus require CPU+GPU hybrid computing (Figure \ref{fig:intro:mem}), where iterating through the entire dataset may incur intolerable time consumption (Figure \ref{fig:intro:time}).
(2) Generalization: 
FL clients usually hold data covering limited domains.
Training on redundant data increases the risk of overfitting, harming generalization to unseen tasks.
Thus, in FL, it is necessary to explore tuning LLMs using fewer representative data to improve efficiency and reduce overfitting, i.e., \emph{data-efficient instruction tuning} \cite{chen2023maybe,sachdeva2024howto}, which has been touched in centralized scenarios, e.g., tuning LLMs with only high-quality samples \cite{zhou2023lima,li2024Superfiltering,lu2023instag,cui2023ultrafeedback} or representative samples (a.k.a. coreset) \cite{chen2023maybe,wu2023self,cao2024instruction}.
Among them, coreset-based methods usually excel in reducing data volume.

Despite their success in centralized scenarios, applying these works to FL faces two main limitations.
\textbf{(L1) Low compatibility with FL:}
Some methods access all data simultaneously \cite{chen2023maybe,wu2023self,cao2024instruction}, violating FL's fundamental principle of restricting direct access to data by third parties, and thus failing to detect inter-client data redundancy.
\textbf{(L2) Suboptimal data representations:} 
The selection of coresets significantly impacts the tuning accuracy \cite{zha2023data,wang2024survey}.
Current methods rely on data representations from the last Transformer layer \cite{chen2023maybe,wu2023self,cao2024instruction}, which may not capture the full spectrum of features to distinguish data samples.

Due to the limited research on federated data-efficient instruction tuning and the shortcomings of centralized methods in FL scenarios, we propose \textbf{\app}, a novel framework of \texbl{fed}erated \texbl{h}ierarchical \texbl{d}ata \texbl{s}election.
It clusters local data on each client to detect \emph{intra-client} data redundancy, and then sends approximate cluster centroids to the server for further clustering to identify \emph{inter-client} data redundancy, cutting down computational costs and enhancing generalization to new tasks by taking the distribution of edge-side data without direct data access, which is compatible with FL (L1). 
Then, we propose to fuse data features from different Transformer layers, in order to provide better data representations for coreset selection (L2).

This work makes the following contributions:
\begin{itemize}[leftmargin=1em]
\item We propose \app, a framework that tunes LLMs using coreset while alleviating \emph{intra-client} and \emph{inter-client} data redundancy. 
To the best of our knowledge, this is the \emph{first} study on federated data-efficient instruction tuning for LLMs.
\item We propose a simple yet effective method to fuse features of varying abstraction levels from different Transformer layers. It works within \app to facilitate clustering-based coreset selection for federated instruction tuning.
\item We conduct experiments on two widely-used instruction datasets with various LLMs and non-IID partitions, showing that \app improves Rouge-L on unseen tasks by 10.72\% on average, using less than 1.5\% of the data samples compared to existing practical federated baselines\footnote{Our codes are available at \url{https://github.com/zhenqincn/FedHDS}.}.
\end{itemize}
\section{Related Work}
\label{sec:related}

\paragraph{Federated Learning for LLM Tuning}
Federated tuning for LLMs gains widespread attention by enabling edge-side data use without direct access.
Due to the scale of LLMs, many studies develop federated tuning based on parameter-efficient fine-tuning (PEFT) techniques to reduce memory and communication costs \cite{zhang2024fedit,babakniya2023slora,zhang2023-FedPETuning,che2023FedPepTAO}. 
Among PEFT techniques, LoRA \cite{lora} gains significant attention, with studies on initialization \cite{babakniya2023slora}, objective consistency \cite{sun2024lora}, and heterogeneous client resources \cite{bai2024federated}.
Some works focus on specific costs, such as communication by transmitting only loss values and random seeds \cite{qin2024full} and memory by employing quantization-aware training \cite{xu2024forward}.

During local training, existing works typically consume all local data \cite{qin2024full,xu2024forward,zhang2024fedit,wu2024fedbiot,kuang2024federatedscope,sun2024lora}, leading to a significant computational cost and overfitting to the local data, as client-side data in FL often consists of many samples from only a few domains.
Unlike existing methods, this work takes a data-centric perspective \cite{zha2023data}, which focuses on \emph{data efficiency} by selecting a small number of representative data samples for LLM tuning, to achieve better efficiency and generalization on unseen tasks.

\paragraph{Data-Efficient LLM Instruction Tuning}
Existing studies show that LLMs learn knowledge primarily from pretraining, and can be taught to produce better outputs with a limited amount of instruction data \cite{zhou2023lima}.
Some studies achieve comparable or better tuning results using fewer samples by focusing on data quality or the representativeness of data samples.

Quality-oriented works either filter out low-quality samples by heuristics \cite{chen2024data}, quality indicators \cite{li2024from,chen2024data} or third-party LLMs \cite{li2024Superfiltering}, or curate high-quality samples by manual efforts \cite{zhou2023lima} or third-party LLMs \cite{lu2023instag,cui2023ultrafeedback}.
Methods based on heuristics or quality indicators
can filter out a limited number of samples.
Manual or third-party LLM-based methods are unsuitable for FL due to privacy risks.

Representativeness-based works select a few representative samples by clustering algorithms, filtering out much of the data while retaining tuning accuracy (e.g., more than 95\% of the dataset) \cite{chen2023maybe,wu2023self}.
However, existing coreset selection methods have shortcomings in FL contexts:
1) the nature of FL, where data stays on the device, prevents existing methods from being fully integrated, such as detecting data redundancy across clients.
2) existing methods often extract data features for clustering using only the last layer of the Transformer \cite{chen2023maybe,wu2023self,li2024from}, which may result in a suboptimal feature space to distinguish coresets. 

\begin{figure*}[t]
  \centering
  \includegraphics[width=0.96\linewidth]{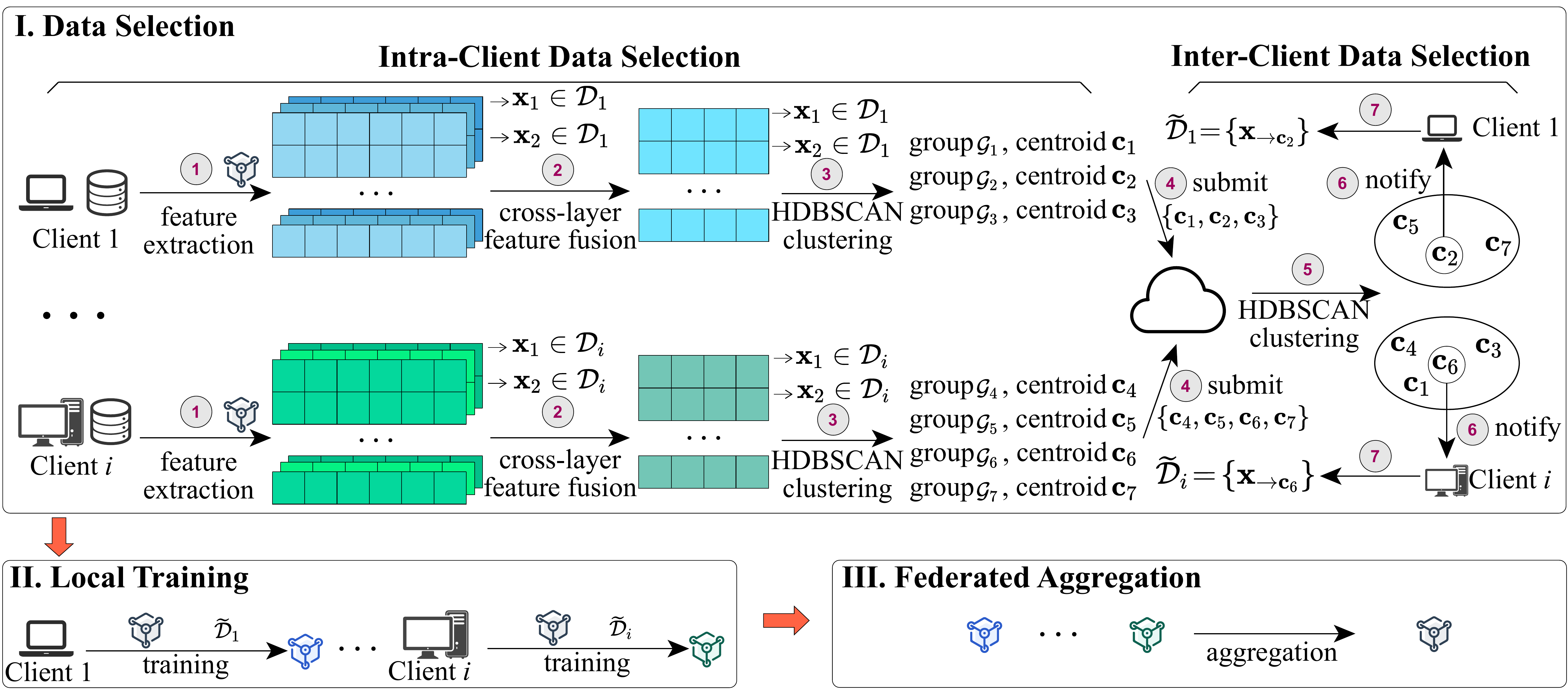}
  \caption{
    Overview of \app in each round of FL.
    It differs from vanilla FL \cite{mcmahan2017fl} by incorporating ``I. Data Selection'' to select representative data samples before ``II. Local Training''.
    Centroid $\mathbf{c}$ is derived from HDBSCAN and does not correspond to specific samples.
    Data sample $\mathbf{x}_{\rightarrow \mathbf{c}_j}$ denotes the one closest to the centroid $\mathbf{c}_j$ of corresponding group $\mathcal{G}_j$.
    Set $\widetilde{\mathcal{D}}_i$ contains representative samples in dataset $\mathcal{D}_i$ of the $i$-th client.
  }
  \label{pic-framework}
\end{figure*}

\paragraph{Summary}
\begin{table}[!h]
\caption{Qualitative comparison between related instruction tuning methods for LLMs and \app.}
\label{tab:qualitative-comparison}
\setlength\tabcolsep{2pt}
\centering
\begin{adjustbox}{max width=\linewidth}
\begin{tabular}{l|c|c|c}
\toprule[1.0pt]
 & \makecell[c]{Reduce Intra-\\Client Data\\Redundancy} & \makecell[c]{Reduce Inter-Client\\Data Redundancy\\(Privacy-Preserving)} & \makecell[c]{Full-Layer\\Feature\\Fusion} \\
\midrule[0.5pt]
\makecell[l]{Federated Instruction\\\ \ Tuning} & \xmark & \xmark & \xmark \\
\midrule[0.5pt]
\makecell[l]{Centralized Data-Efficient\\\ \ Instruction Tuning} & \cmark & \xmark & \xmark \\
\midrule[0.5pt]
\textbf{\app (ours)} & \cmark & \cmark & \cmark \\
\bottomrule[1.0pt]
\end{tabular}
\end{adjustbox}
\end{table}

As Table \ref{tab:qualitative-comparison}, existing federated tuning overlooks data redundancy.
Centralized data-efficient methods are either impractical for FL due to external data transmission or fail to address inter-client redundancy (L1).
They also rely on only the last-layer Transformer outputs to distinguish data samples, potentially causing sub-optimal data representations (L2).
Unlike these methods, \app performs instruction tuning using a coreset while alleviating both intra- and inter-client redundancy, with fused features from all Transformer layers.
\section{Problem Formulation}
\label{sec:problem-formulation}
\paragraph{Federated Instruction Tuning}
Assuming there are $N$ clients in an FL system, each client $i$ holds a private dataset $\mathcal{D}_i$ containing several instances of instruction data.
Given an LLM $\mathbf{w}\in \mathbb{R}^d$ initialized by the pretrained weight $\mathbf{w}^0$, federated instruction tuning aims at optimizing $\mathbf{w}$ with the private data held by clients towards the following objective,
\begin{equation}
  \min_{\mathbf{w}} f(\mathbf{w}) \triangleq \sum_{i=1}^N \lambda_i \cdot \mathbb{E}_{\mathbf{x} \sim \mathcal{D}_i}\left[ \mathcal{L}(\mathbf{w}; \mathbf{x})\right],
  \label{eq-fl-optimization}
\end{equation}
where $\mathcal{L}(\mathbf{w}; \mathbf{x})$ is the loss evaluated on model $\mathbf{w}$ for data instance $\mathbf{x}$ sampled from $\mathcal{D}_i$, and $\lambda_i$ is the weight of client $i$ that follows $\lambda_i > 0$ and $\sum_i^N \lambda_i=1$. 
Symbol $\mathbf{x}$ is used as the batch size is 1 to reduce memory usage \cite{qin2024full}.
To solve Eq. \eqref{eq-fl-optimization}, FL iterates multiple rounds. 
In each round $r$, several active clients get the latest model parameters $\mathbf{w}^r$ from the server and perform several steps of stochastic gradient descent (SGD), as
\begin{equation}
  \mathbf{w}^r_{i, t + 1}= \mathbf{w}^r_{i, t} - \eta \cdot \nabla_{\mathbf{w}^r_{i, t}} \mathcal{L}(\mathbf{w}^r_t; \mathbf{x}), \forall \mathbf{x} \in \mathcal{D}_i,
  \label{eq-local-sgd}
\end{equation}
where $\mathbf{w}^r_{i, t}$ is the model of client $i$ at the $t$-th local step in round $r$, and $\eta$ is the learning rate.
Typically, the process iterates over the local dataset $\mathcal{D}_i$ for one or more epochs \cite{qin2024full,xu2024forward,zhang2024fedit}.
After local training, each active client sends the updated model to the server.
To alleviate communication and memory costs, FL typically adopts PEFT techniques, where only a small subset of model parameters is trained and transmitted.
Similarly to \citet{zhang2024fedit}, this work only trains and transmits LoRA adapters.

\paragraph{Federated Data-Efficient Instruction Tuning}
Assuming each client uses a data selection function $f(\mathbf{x}; \cdot) \rightarrow \left\{0, 1\right\}$ to construct a coreset
\begin{equation}
    \widetilde{\mathcal{D}}_i = \left\{\mathbf{x} \in \mathcal{D}_i \mid f(\mathbf{x} ; \cdot) = 1\right\},
\end{equation}
and performs local training only on $\widetilde{\mathcal{D}}_i$ as Eq. \eqref{eq-local-sgd}.
If $f$ makes the model $\widetilde{\mathbf{w}}$ trained on $\{\widetilde{\mathcal{D}}_i\}_{i=1}^{N}$ achieve accuracy comparable to—or even better than—that obtained with the full datasets, while satisfying
\begin{equation}
    \sum_{i=1}^N \left| \widetilde{\mathcal{D}}_i\right| /\ \sum_{i=1}^N \left| \mathcal{D}_i\right| \ll 1,
\end{equation}
then the method is \emph{data-efficient}, which greatly reduces the computation cost than the full-data ones.

Additionally, by considering the redundancy among local coresets, we can further reduce the local iterations for each client.
Assuming an ideal distribution $\mathcal{P}^*$ that each real-world data $\mathbf{x}$ follows \cite{qin2024synergy}, and data samples in each coreset $\widetilde{\mathcal{D}}_i$ follow $\mathcal{P}_i$, where local data distributions $\{\mathcal{P}_1, \ldots, \mathcal{P}_N \}$ share mutual similarity to some extent since all the client-side data can be regarded as sampled from $\mathcal{P}^*$, so that inter-client data redundancy may exist.
Experimental results in Section \ref{subsec-exp-ablation} also illustrate this.
To alleviate intra-client and inter-client data redundancy, we design \app.
\section{Approach}
\label{sec-approach}
\subsection{Overview}
\label{subsec-app-overview}
\app identifies the representative data samples with their latent features.
As shown in Figure \ref{pic-framework}, in round $r$, after downloading the global model $\mathbf{w}^r$, each client performs intra-client selection, followed by an inter-client selection performed by the server.

In intra-client selection, client $i$ gets the hidden state in each Transformer layer with each data sample $\mathbf{x} \in \mathcal{D}_i$ as the input (\ding{192} in Figure \ref{pic-framework}).
Then, these features are fused across different Transformer layers (\ding{193} in Figure \ref{pic-framework}).
Next, clustering is performed to partition the data into several groups (\ding{194} in Figure \ref{pic-framework}).
Each group holds an approximate centroid $\mathbf{c}$ that does not correspond to an individual data point.
After that, the client sends these centroids $\{\mathbf{c}_1, \mathbf{c}_2, \ldots, \}$ to the server (\ding{195} in Figure \ref{pic-framework}).

Upon receiving the centroids from all active clients, inter-client selection starts.
The server clusters received centroids into several groups (\ding{196} in Figure \ref{pic-framework}).
In each group, the point closest to its centroid $c^{\text{II}}_j$ is designated as the chosen one.
Then, the server notifies each client regarding which of their sent centroids are selected (\ding{197} in Figure \ref{pic-framework}). 
Next, each client $i$ adds the data sample closest to each of the selected centroid within the corresponding group to coreset $\widetilde{\mathcal{D}}_i$ (\ding{198} in Figure \ref{pic-framework}).

Finally, client $i$ performs local training on $\widetilde{\mathcal{D}}_i$ as \citet{zhang2024fedit}.
The selection processes are summarized in Algorithm \ref{algo:data-selection} of Appendix \ref{sec:appendix:algo}.
In the following, we detail the design of \app.

\subsection{Intra-Client Data Selection}
\label{subsec-app-intra}
Data samples are selected based on features.
Given an LLM with $l$ Transformer layers, it can extract data features layer by layer and token by token, as
\begin{equation}
\begin{bmatrix}
 \mathbf{h}^{1, 1}_j & \mathbf{h}^{1, 2}_j & \ldots  & \mathbf{h}^{1, -1}_j \\
 \mathbf{h}^{2, 1}_j & \mathbf{h}^{2, 2}_j & \ldots  & \mathbf{h}^{2, -1}_j \\
 \ldots & \ldots & \ldots & \ldots\\
 \mathbf{h}^{l, 1}_j & \mathbf{h}^{l, 2}_j & \ldots  & \mathbf{h}^{l, -1}_j \\
\end{bmatrix},
\label{eq-transformer-feature-matrix}
\end{equation}
where $\mathbf{h}^{l, b}_j \in \mathbb{R}^v$ is the hidden states of the $b$-th token from the $l$-th Transformer layer for the $j$-th data sample, with $b$=-1 denoting the last token.
From the token level, \app uses the hidden state of the last token as \citet{li2024from} since it encapsulates all preceding token information.

\begin{figure}[t]
  \centering
  \subfigure[CH Index $\uparrow$]{
    \includegraphics[width=0.45\linewidth]{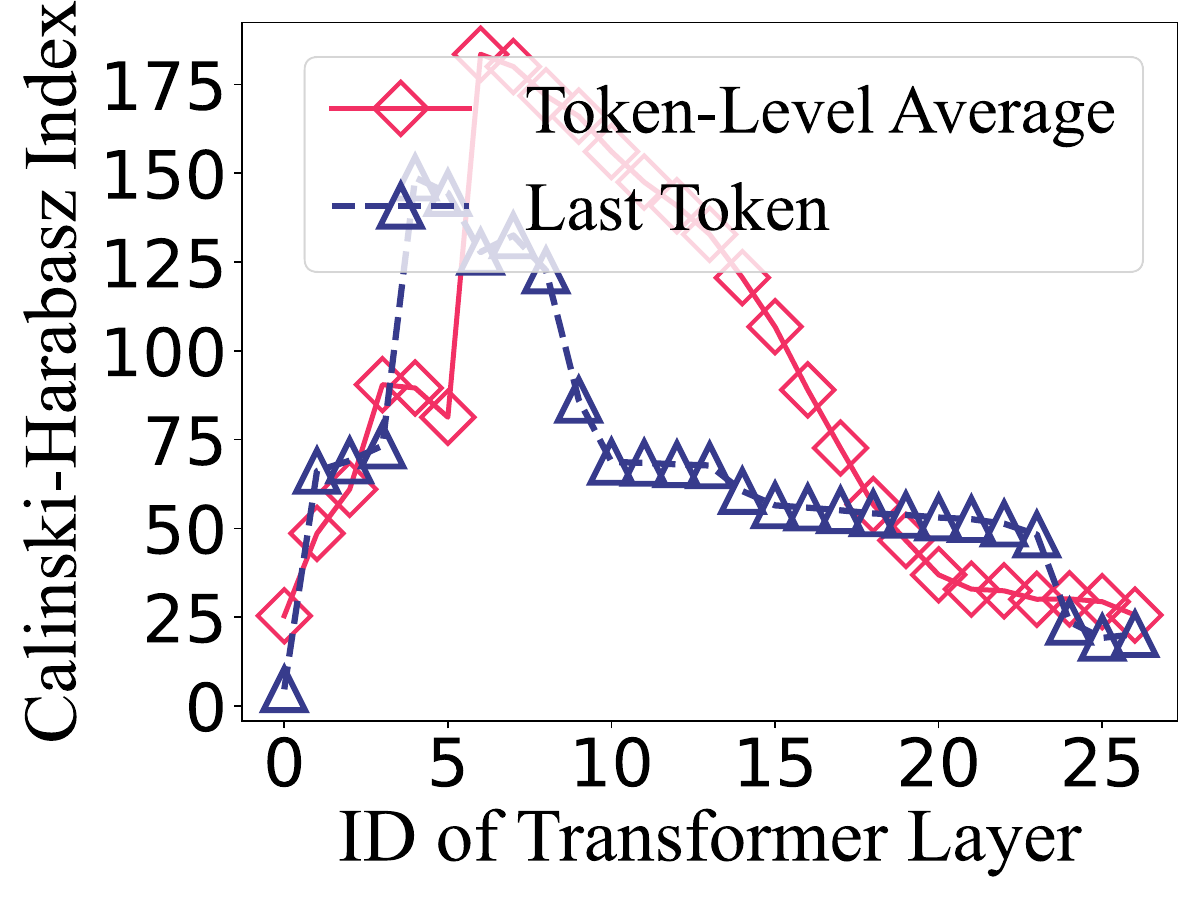}
    \label{subfig-clustering-ch}
  }
  \subfigure[F$_{1}$-Score $\uparrow$]{
    \includegraphics[width=0.45\linewidth]{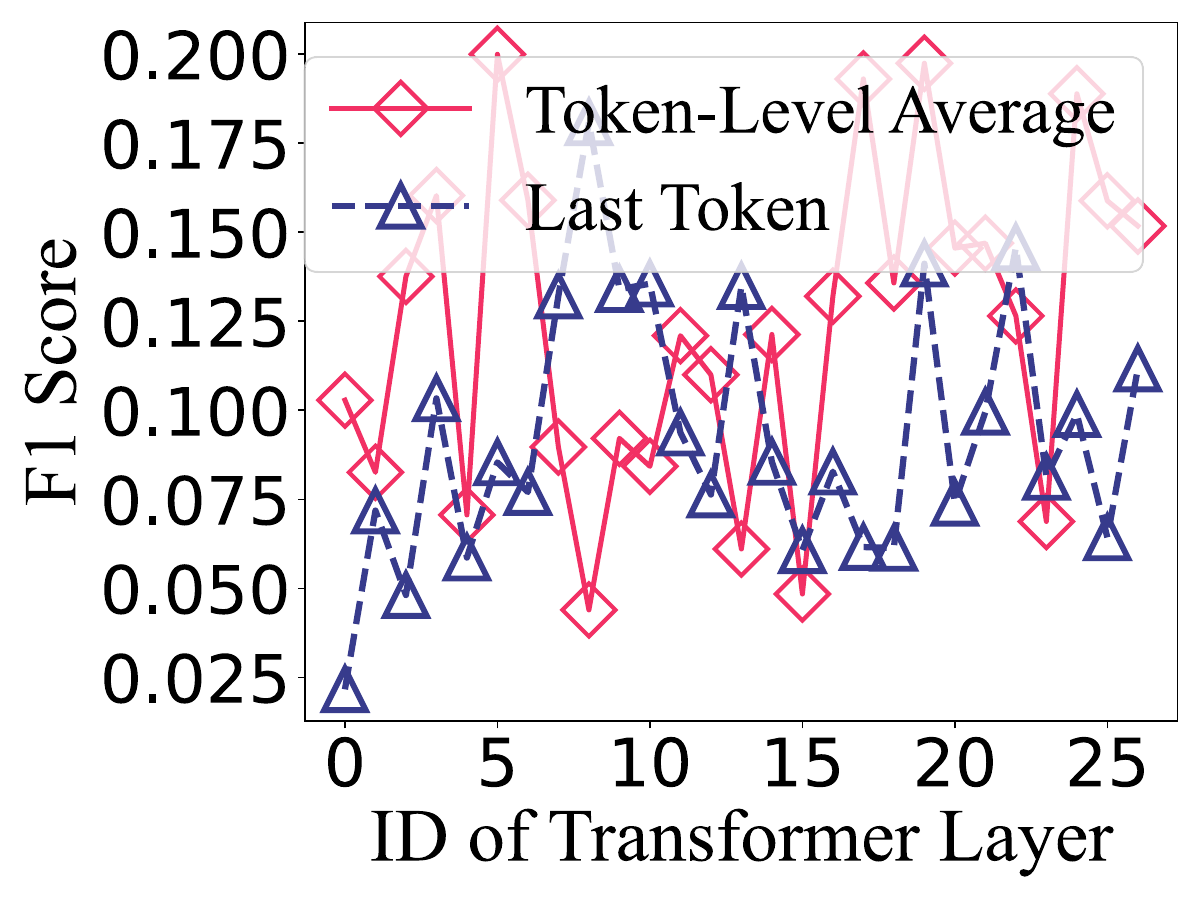}
    \label{subfig-clustering-f1}
  }
  \caption{Evaluations of clustered data groups based on features from different Transformer layers, obtained in a centralized scenario with \modelllama on \datadolly.
  }
  \label{pic-clustering-metrics-layer}
\end{figure}
\begin{figure}[t]
  \centering
  \subfigure[The last layer.]{
    \includegraphics[width=0.41\linewidth]{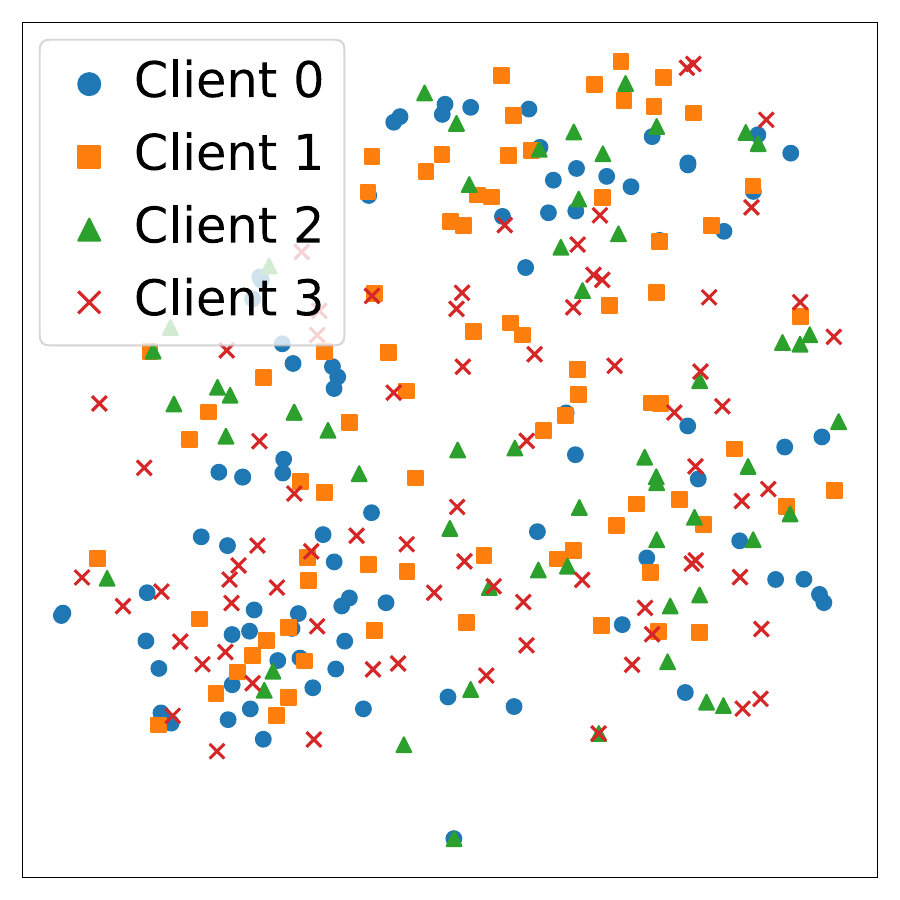}
    \label{subfig-vis-1B-dolly5.0-last}
  }
  \hspace{0.1cm}
  \subfigure[Fused from all layers.]{
    \includegraphics[width=0.41\linewidth]{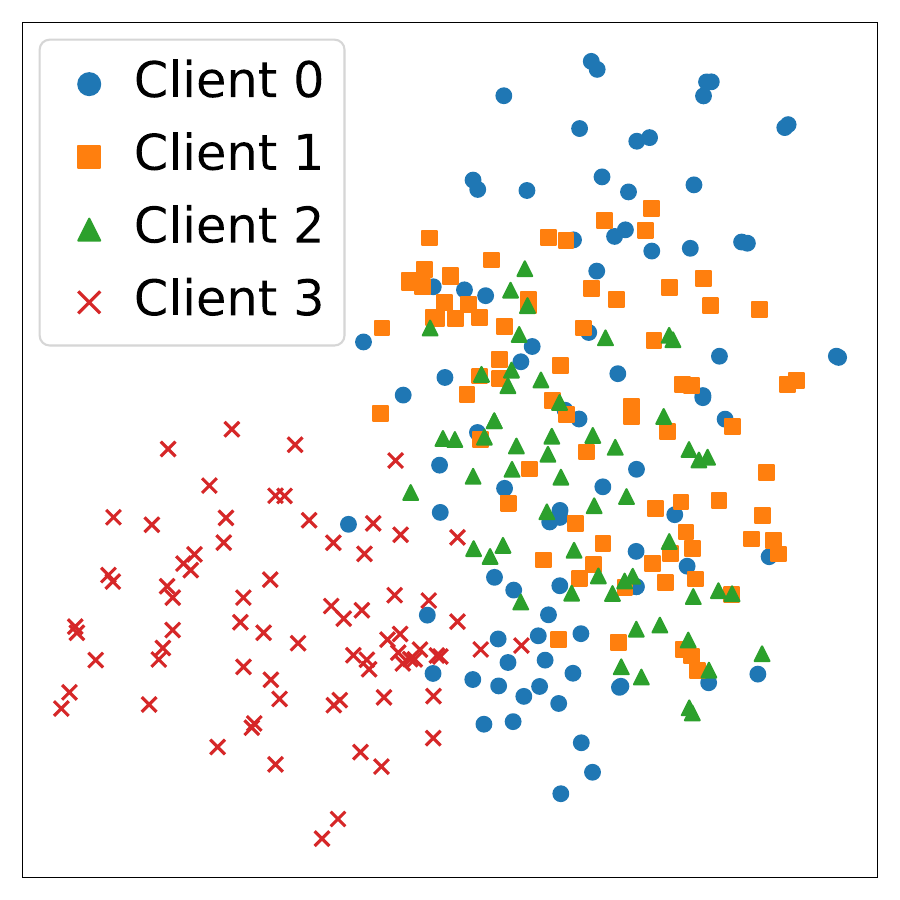}
    \label{subfig-vis-1B-dolly5.0-fusion}
  }
  \caption{Visualization of features obtained with \modeldatajuicer on \datadolly ($\alpha$=5.0).
  }
  \label{pic-visualization-feature}
\end{figure}
At the layer level, prior works often use the last layer \cite{chen2023maybe,wu2023self,li2024from}, which may not be universally optimal since different layers provide varying abstraction degrees to data representations.
We show it with a toy example on \datadolly, clustering its 8-category instructions into 8 groups using K-means.
Figure \ref{pic-clustering-metrics-layer} evaluates clustering quality with Calinski-Harabasz Index \cite{calinski1974dendrite}, and F$_1$-score (more evaluations are left in Appendix \ref{sec:appendix:additional-exp-metrics-layer}), showing that the last layer is not universally optimal, and no single layer excels across all metrics.

Predicting the optimal layer is challenging, and computing clustering metrics for all layers is costly. 
A feasible approach is fusing features from all layers.
One naive method is concatenating the last token's hidden states across layers, as
\begin{equation}
\mathbf{h}_j = \left [\mathbf{h}^{1, -1}_{j}, \mathbf{h}^{2, -1}_{j}, \ldots, \mathbf{h}^{l, -1}_{j}  \right ].
\label{eq-feature-concatenate}
\end{equation}
However, Figure \ref{pic-clustering-metrics-layer} shows that some layers degrade data separability.
Dimensionality reduction may alleviate the impact of inappropriate dimensions, e.g., low-variance dimensions have minimal impact on distance calculation in t-SNE \cite{van2008visualizing}.
Thus, we apply $P: \mathbb{R}^{l \times v} \to \mathbb{R}^k, k \ll l \times v$ to fuse the features from $l$ layers, as
\begin{equation}
\!\!\!\!
\resizebox{0.9\linewidth}{!}{$
  \{\widetilde{\mathbf{h}}_1, \widetilde{\mathbf{h}}_2, \ldots, \widetilde{\mathbf{h}}_{\left|\mathcal{D}_i\right|} \} = P(\{\mathbf{h}_1, \mathbf{h}_2, \ldots, \mathbf{h}_{\left|\mathcal{D}_i\right|} \}).
$}
\label{eq-dimension-reduction}
\end{equation}
We choose t-SNE with Barnes-Hut implementation \cite{van2013barnes}, as it is more effective in nonlinear spaces than earlier methods such as PCA \cite{jolliffe2002principal}.
The fused feature dimension $k$ is set to 2 for efficiency.
Figure \ref{pic-visualization-feature} shows the effectiveness of feature fusion, where a random subset of \datadolly is partitioned to 4 clients with label distribution skew.
If last-layer features are employed, data among clients tends to be scattered, while fused features form relatively clear boundaries between clients, enhancing sample distinction.

After obtaining fused features of local data $\mathcal{D}_i$, denoted by $\{\widetilde{\mathbf{h}}_1, \widetilde{\mathbf{h}}_2, \ldots, \widetilde{\mathbf{h}}_{\left|\mathcal{D}_i\right|} \}$, HDBSCAN \cite{campello2013density} is applied to cluster them into groups $\{\mathcal{G}_1, \mathcal{G}_2, \ldots\}$.
Each group $\mathcal{G}_j$ holds a centroid $\mathbf{c}_j$ that does not correspond to a real sample.
These centroids are sent to the server to determine which groups are selected for tuning.

\subsection{Inter-Client Data Selection}
\label{subsec-app-inter}
As discussed in Section \ref{sec:problem-formulation} and shown in Figure \ref{pic-visualization-feature}, there may be similarities among client-side data.
Thus, we cluster the approximate centroids sent from the clients to the server with HDBSCAN to filter redundant data groups among clients, as
\begin{equation}
\!\!\!\!
\resizebox{0.9\linewidth}{!}{$
  \mathbb{G}^{\text{II}} = \left\{\mathcal{G}^{\text{II}}_1, \mathcal{G}^{\text{II}}_2, \ldots\right\} = \text{HDBSCAN}(\left\{\mathbf{c}_1, \mathbf{c}_2, \ldots \right\}).
\label{eq-layer-2-clustering}
$}
\end{equation}
For each $\mathcal{G}^{\text{II}}_j$ with its approximate centroid $\mathbf{c}^{\text{II}}_j$, the server identifies the first-layer group $\mathcal{G}_s$ whose centroid $\mathbf{c}_s$ closest to $\mathbf{c}^{\text{II}}_j$ as the selected one, as
\begin{equation}
\!\!\!\!
\resizebox{0.9\linewidth}{!}{$
  \mathbb{G}^{\text{selected}} = \{\mathcal{G}_s \mid \mathbf{c}_s = \underset{\mathbf{c} \in \mathcal{G}^{\text{II}}_j}{\operatorname{arg \ min}} \| \mathbf{c} - \mathbf{c}^{\text{II}}_j \| \}^{\mathcal{G}^{\text{II}}_j \in \mathbb{G}^{\text{II}}}.
$}
\label{eq-layer-2-selection}
\end{equation}
Then, the corresponding clients will be notified of the selection of first-layer groups. 
Given the partitioned data groups $\mathbb{G}_i = \left\{\mathcal{G}_1, \mathcal{G}_2, \ldots\right\}$ of client $i$, the subset finally used for tuning $\widetilde{\mathcal{D}}_i$ is obtained as:
\begin{equation}
\!\!\!\!
\resizebox{0.9\linewidth}{!}{$
  \widetilde{\mathcal{D}}_i = \{\mathbf{x} \mid q = \underset{\mathbf{x} \in \mathcal{G}_j}{\operatorname{arg \ min}} \left\| \mathbf{x} - \mathbf{c}_j\right\| \}^{\mathcal{G}_j \in \mathbb{G}_i, \mathcal{G}_j \in \mathbb{G}^{\text{selected}}}.
$}
\label{eq-tuning-set-construction}
\end{equation}
In each selected group $\mathcal{G}_j$, the data sample closest to centroid $\mathbf{c}_j$ is added in $\widetilde{\mathcal{D}}_i$ ($\mathbf{x}_{\rightarrow \mathbf{c}_j}$ in Figure \ref{pic-framework}).

\subsection{Instruction Tuning with Coresets}
\label{subsec-app-tuning}
After determining coresets, each client $i$  performs local training using only $\widetilde{\mathcal{D}}_i$ (II in Figure \ref{pic-framework}), as
\begin{equation}
  \mathbf{w}^r_{i, t + 1}= \mathbf{w}^r_{i, t} - \eta \cdot \nabla_{\mathbf{w}^r_{i, t}} \mathcal{L}(\mathbf{w}^r_t; \mathbf{x}), \forall \mathbf{x} \in \widetilde{\mathcal{D}}_i,
  \label{eq-local-sgd-subset}
\end{equation}
followed by sending the optimized model parameters.
Then, the server performs FedAvg \cite{mcmahan2017fl} on the received parameters, as
\begin{equation}
    \mathbf{w}^{r + 1} = \frac{1}{\sum_{i \in \mathbb{M}_r} | \widetilde{\mathcal{D}}_i |} \sum_{i \in \mathbb{M}_r} | \widetilde{\mathcal{D}}_i | \cdot \mathbf{w}^{r}_{i, -1},
    \label{eq-aggregation}
\end{equation}
where $\mathbb{M}_r$ contains the indices of activate clients in the $r$-th round of FL.
Finally, the next round starts.

\app only performs downsampling on the local training data based on FedAvg, thus its convergence is theoretically supported, as discussed in Appendix \ref{sec:appendix:convergence-analysis}.
Besides, we provide discussions on the theoretical speedup of \app in Appendix \ref{sec:appendix:speed-up}.

\subsection{Further Enhancement for \app}
\label{subsec-app-enhancement}
\paragraph{Efficiency}
\label{subsubsec-app-enhancement-efficiency}
The efficiency of \app can be further improved by faster feature extraction, since it requires forward propagation across all the local data with the LLM $\mathbf{w}$. 
Motivated by studies on retrieval augmented generation which adopt a light-weight retrieval model to extract data features \cite{fan2024survey}, we use a smaller language model sharing Transformer architecture and and next-token prediction paradigm with $\mathbf{w}$ as a proxy to generate representations for each data, i.e., a small version of GPT-2 (about 124M parameters).
We term this approach as \apppro.

\paragraph{Privacy}
\label{subsubsec-app-enhancement-privacy}
Although local centroids sent to the server are two-dimensional vectors that do not correspond to real samples, \app offers more information than vanilla FL. 
We can adopt a straightforward differential privacy method for scenarios requiring stronger privacy protection \cite{hu2022membership}.
We scale elements in local centroids to [-1, 1] with \texttt{tanh} to maintain data distinguishability despite extreme values.
Gaussian noise is then added to the scaled centroids before transmission, deriving their differential privacy (Theorem \ref{thm:privacy}).
Experiments in Section \ref{subsec:exp:privacy} show that \app performs well with a reasonable noise scale.
\begin{theorem}\label{thm:privacy}
  Let $\mathbf{c} \in [-1,1]$ be the original centroid and $\mathbf{c}' = \mathbf{c} + \mathbf{z}$ be the one noised by Gaussian noise $\mathbf{z} \sim \mathcal{N}(0, \sigma^2 \mathbf{I})$ with standard deviation $\sigma$. Suppose $\varepsilon, \delta \in (0,1)$. Noised centroid $\mathbf{c}'$ satisfies $(\varepsilon, \delta)$-differential privacy if $\sigma \geq \frac{2\sqrt{2\log(1.25/\delta)}}{\varepsilon}$.
\end{theorem}
\begin{proof}
  The proof follows directly by applying the Gaussian mechanism~\cite{dwork2014algorithmic}, where the sensitivity (i.e., range of $\mathbf{c}$) is 2.
\end{proof}
\section{Evaluations}
\label{sec-exp}

\subsection{Experimental Setup}
\label{subsec:exp-setup}
\paragraph{Baselines}
We introduce six federated tuning methods using full data as baselines:
1) FedAvg that tunes and transmits the full LLM, included for reference due to high costs;
2\&3) FedPTuning and FedPrompt \cite{kuang2024federatedscope} that apply PEFT techniques of P-Tuning \cite{liu2023gpt} and Prompt Tuning \cite{lester2021power} based on FedAvg, respectively, trained with Adam \cite{adam} optimizer; 
4\&5) FedIT \cite{zhang2024fedit}: instruction tuning based on FedAvg with LoRA, optimized with Adam or SGD (FedIT-SGD);
6) FlexLoRA \cite{bai2024federated} that supports LoRA adapters with varying ranks based on FedIT.

Given the lack of data-efficient works in FL, we develop two federated methods and a centralized one using coresets:
1) Random: It randomly selects a ratio of local data. Although being native, it is a strong baseline \cite{lin2024data,sachdeva2024howto} by preserving the original data distributions;
2) Perplexity that selects data with lower perplexity scores \cite{chen2024data};
and
3) Coreset-Cent \cite{chen2023maybe} that selects a fixed ratio of data by K-means clustering on the last-layer features.

\paragraph{Datasets and Evaluations}
Following \citet{qin2024full,kuang2024federatedscope}, we conduct experiments on \datani \cite{supernaturalinstructions} (NI) and \datadolly \cite{DatabricksBlog2023DollyV2}, and employ Rouge-L on held-out tasks as the evaluation metrics.
After preprocessing (detailed in Appendix \ref{subsec:appendix:reproducibility-data-preprocessing}), NI contains 738 training tasks, each of which is assigned to a unique client, providing non-IIDness with feature skew, and the natively provided 119 test tasks are used for evaluation. 
\datadolly contains 8 tasks.
The last one is used for evaluation, and the rest are partitioned to 200 clients via Dirichlet distribution with $\alpha$ set to 0.5 and 5.0, respectively.
Experiments on these two datasets provide scenarios where the client has hundreds and dozens of data samples, respectively.

\paragraph{Implementation}
\label{subsec-exp-implementation}
This work targets cross-device FL, thus, 5\% of the clients are randomly selected to participate in each round.
Limited by space, the implementation is detailed in Appendix \ref{sec:appendix:reproducibility}.

\subsection{Comparison on Accuracy}
\label{subsec:exp-main}
\begin{table*}[t]
\renewcommand\arraystretch{0.92}
\caption{Rouge-L (\%) comparisons. 
Parentheses indicate the ratio of consumed data samples compared to full-data methods.
Each value is the average Rouge-L obtained in the last round of four runs with different random seeds. 
\hl{Coreset-Cent} and \hl{FedAvg} are introduced just as references as they are not practical to end devices.
Bold and underlined numbers are the best and second-best values among approaches practical to cross-device FL, respectively.
}
\label{tab-performance}
\setlength\tabcolsep{3.3pt}
\centering
\begin{adjustbox}{max width=\textwidth}
\begin{tabular}{l|cc|cc|cc}
\toprule[1.0pt]
\multirow{2}{*}{Approach}     & \multicolumn{2}{c|}{\textbf{\datani} (Meta Non-IID)}              & \multicolumn{2}{c|}{\textbf{\datadolly ($\alpha=0.5$)}} & \multicolumn{2}{c}{\textbf{\datadolly ($\alpha=5.0$)}} \\
\cmidrule{2-7}  
                              & \modeldatajuicer                & \modelllama                     & \modeldatajuicer                & \modelllama                      & \modeldatajuicer                & \modelllama \\
\midrule[1.0pt]
\rowcolor{gray!20}Coreset-Cent    & 31.36 \std{0.80} \pdata{1.00} & 34.81 \std{0.90} \pdata{0.01}                 & 33.27 \std{0.33} \pdata{0.50} & 35.48 \std{1.08} \pdata{1.00}               & 33.27 \std{0.33} \pdata{0.50} & 35.48 \std{1.08} \pdata{1.00}   \\
\rowcolor{gray!20}FedAvg & 22.08 \std{1.52} \pdata{100} & 27.88 \std{0.75} \pdata{100}               & 32.30 \std{1.23} \pdata{100} & 34.27 \std{0.45} \pdata{100}                   & 33.38 \std{1.43} \pdata{100} & 33.95 \std{0.79} \pdata{100} \\
\midrule[1.0pt]
FedPTuning                    & 19.61 \std{2.71} \pdata{100}    & 25.41 \std{1.14} \pdata{100}               & 23.98 \std{3.23} \pdata{100} & 30.30 \std{1.16} \pdata{100}                   & 25.33 \std{2.48} \pdata{100} & 29.08 \std{1.33} \pdata{100} \\
FedPrompt                     & \ \ 6.04 \std{0.12} \pdata{100} & \ \ 8.95 \std{2.47} \pdata{100}            & 32.73 \std{0.87} \pdata{100} & 24.50 \std{4.78} \pdata{100}       & 32.51 \std{1.31} \pdata{100} & 23.94 \std{4.15} \pdata{100} \\
FedIT-SGD                     & 19.40 \std{1.83} \pdata{100}    & 28.14 \std{0.85} \pdata{100}               & 27.23 \std{0.68} \pdata{100} & 29.28 \std{0.50} \pdata{100}                        & 27.28 \std{1.35} \pdata{100}  & 29.19 \std{0.89} \pdata{100} \\
FlexLoRA                & 23.19 \std{2.14} \pdata{100} & 28.86 \std{0.55} \pdata{100} & 29.81 \std{1.06} \pdata{100} & 32.84 \std{0.99} \pdata{100} & 29.17 \std{1.35} \pdata{100} & 32.18 \std{1.28} \pdata{100} \\
FedIT                         & 22.30 \std{0.42} \pdata{100} & 28.13 \std{0.50} \pdata{100}    & 30.80 \std{0.98} \pdata{100} & 33.23 \std{1.51} \pdata{100}            & 30.97 \std{0.43} \pdata{100} & 33.68 \std{1.07} \pdata{100} \\
\midrule[0.3pt]
Random                        & \texl{26.20 \std{1.71}} \pdata{0.20}   & 31.23 \std{1.37} \pdata{2.00}              & 32.59 \std{0.10} \pdata{1.50}  & 33.82 \std{0.82} \pdata{1.50}        & 32.24 \std{0.43} \pdata{1.50} & 34.29 \std{0.85} \pdata{5.00} \\
Perplexity                    & 24.45 \std{0.77} \pdata{5.00}   & 30.49 \std{0.21} \pdata{5.00}              & 32.73 \std{0.15} \pdata{1.50}       & 33.71 \std{0.51} \pdata{1.50}        & 32.24 \std{0.22} \pdata{1.50} & 33.88 \std{0.34} \pdata{5.00} \\
\midrule[1.0pt]
\app                         & \texb{26.64 \std{0.79}} \pdata{0.20} & \texl{32.32 \std{0.92}} \pdata{0.15}                   & \textbf{33.38 \std{0.40}} \pdata{0.82} & \textbf{35.40 \std{0.78}} \pdata{1.18}                             & \textbf{33.70 \std{0.19}} \pdata{0.88} & \textbf{35.79 \std{0.43}} \pdata{1.34}  \\
\apppro                      & 25.93 \std{0.75} \pdata{0.23} & \textbf{32.93 \std{0.64}} \pdata{0.22}           & \texl{33.28 \std{0.44}} \pdata{1.31} & \texl{35.01 \std{0.65}} \pdata{1.26}         & \texl{33.52 \std{0.20}} \pdata{1.19} &  \texl{35.42 \std{0.29}} \pdata{1.28} \\
\bottomrule[1.0pt]
\end{tabular}
\end{adjustbox}
\end{table*}
\newcommand{\gpu}{\footnotesize{BP w/ GPU}}
\newcommand{\hybirdbp}{\footnotesize{BP w/ CPU+GPU}}
\newcommand{\pbar}[1]{
  \begin{tikzpicture}
    \draw[fill=blue!50] (0,0) rectangle (#1*0.5,0.2);
    \draw (0,0) rectangle (0.5,0.2);
  \end{tikzpicture}
}

\begin{table*}[t]
\renewcommand\arraystretch{0.92}
\caption{
Comparisons on 1) client-side time (CTime) i.e., the time a client spends performing local computations during a round of FL, 2) the overall time consumption across all rounds (total time), and 3) the speedup ratio. Time is calculated with training using CPU+GPU due to limited GPU memory on edge devices.
}
\label{tab-speedup}
\setlength\tabcolsep{1.8pt}
\centering
\begin{adjustbox}{max width=\textwidth}
\begin{tabular}{l|r@{\hspace{0.2cm}}rr|r@{\hspace{0.2cm}}rr|r@{\hspace{0.2cm}}rr|r@{\hspace{0.2cm}}rr}
\toprule[1.0pt]
\multicolumn{1}{c|}{\multirow{3}{*}{Approach}} & \multicolumn{3}{c|}{\textbf{\modeldatajuicer on NI}} & \multicolumn{3}{c|}{\textbf{\modelllama on NI}}             & \multicolumn{3}{c|}{\textbf{\modeldatajuicer on \datadolly}} & \multicolumn{3}{c}{\textbf{\modelllama on \datadolly}} \\
\cmidrule{2-13}  
                         & \multicolumn{1}{c}{\small CTime} & \multicolumn{1}{c}{\small Total Time} & \multicolumn{1}{c|}{\small Speedup} & \multicolumn{1}{c}{\small CTime} & \multicolumn{1}{c}{\small Total Time} & \multicolumn{1}{c|}{\small Speedup} & \multicolumn{1}{c}{\small CTime} & \multicolumn{1}{c}{\small Total Time} & \multicolumn{1}{c|}{\small Speedup} & \multicolumn{1}{c}{\small CTime} & \multicolumn{1}{c}{\small Total Time} & \multicolumn{1}{c}{\small Speedup} \\
\midrule[1.0pt]
FedIT      & 489S & 8D9H\pbar{1.0} & 1$\times$      & 811S & 13D21H\pbar{1.0} & 1$\times$& 56.0S & 9H20M\pbar{1.0} & 1$\times$ & 114S & 19H6M\pbar{1.0} & 1$\times$\\
Random     & 13.2S & 5H24M\pbar{0.0269} & 37.2$\times$& 35.3S & 14H29M\pbar{0.043} & 23.0$\times$& 6.5S & 43M25S\pbar{0.078} & 12.89$\times$ & 25.2S & 4H12M\pbar{0.22} & 4.54$\times$\\
Perplexity & 46.6S & 19H9M\pbar{0.0953} & 10.5$\times$& 85S & 1D11H\pbar{0.105} & 9.51$\times$& 8.9S & 59M27S\pbar{0.106} & 9.42$\times$ & 27.5S & 4H34M\pbar{0.239} & 4.18$\times$\\
\app       & 27.5S & 11H17M\pbar{0.0561} & 17.8$\times$& 40.7S & 16H42M\pbar{0.05} & 19.9$\times$& 5.2S & 34M52S\pbar{0.062} & 16.06$\times$& 14.6S & 2H25M\pbar{0.127} & 7.90$\times$\\
\apppro    & 10.0S & 4H7M\pbar{0.0205} & 48.8$\times$& 20.1S & 8H15M\pbar{0.025} & 40.4$\times$& 4.5S & 29M42S\pbar{0.053} & 18.86$\times$& 17.2S & 2H52M\pbar{0.15} & 6.66$\times$\\
\bottomrule[1.0pt]
\end{tabular}
\end{adjustbox}
\end{table*}
We compare these methods in Table \ref{tab-performance}.
The results of full-data FL methods are derived from \citet{qin2024full} under the same settings, and those of others are obtained within the best hyperparameters.

\paragraph{Comparison to Full-Data Methods}
From Table \ref{tab-performance}, \app and \apppro outperform full-data FL baselines across the six scenarios with consumed data samples less than 1.5\% of them.
Particularly, on NI with \modeldatajuicer, \app and \apppro relatively improve Rouge-L over FedIT—the practical full-data FL baseline achieving the best average accuracy—by 19.5\% and 16.3\%, respectively.
Averaged across the six scenarios, \apppro improves the Rouge-L score relative to FedIT by 10.72\%.
Besides, compared to FedIT, \app achieves an average improvement of 4.26\% and 4.96\% in Rouge-L on \datadolly when $\alpha$=0.5 and $\alpha$=5.0, respectively, indicating that \app performs better when client-side data has certain similarity.
In a cross-device FL scenario with a large scale of clients, it is common for different clients to have similar data distributions. 
These results demonstrate the effectiveness of our approaches for improving generalization.

\paragraph{Comparison to Coreset Baselines}
Both \app and \apppro outperform Random and Perplexity in 5 of the 6 scenarios. 
Only on NI with \modeldatajuicer, Random slightly outperforms \apppro.
From \citet{lin2024data,sachdeva2024howto}, Random is a strong baseline as it preserves the data distribution. 
However, it is affected by data ratios \cite{cao2024instruction}, and determining the optimal data ratio requires extensive experimentation.
Differently, \app and \apppro can automatically determine an appropriate data ratio.
Note that in 7 out of the 12 scenarios involving Random and Perplexity, the data ratios with the best Rouge-L are inspired by our approaches.
Even so, our approaches outperform them in the vast majority of cases, highlighting the need for a well-designed data selection strategy.

The centralized method, Coreset-Cent, surpasses our approaches on the complex dataset, NI, indicating room for further improvement in FL methods.

\begin{table*}[t]
\caption{
Rouge-L (\%) comparisons for ablation studies, organized in the same manner of Table \ref{tab-performance}.
}
\label{tab-ablation}
\setlength\tabcolsep{3.2pt}
\centering
\begin{adjustbox}{max width=\textwidth}
\begin{tabular}{l|cc|cc|cc}
\toprule[1.0pt]
\multirow{2}{*}{Approach}     & \multicolumn{2}{c|}{\textbf{\datani} (Meta Non-IID)}              & \multicolumn{2}{c|}{\textbf{\datadolly ($\alpha=0.5$)}} & \multicolumn{2}{c}{\textbf{\datadolly ($\alpha=5.0$)}} \\
\cmidrule{2-7}  
                              & \modeldatajuicer                & \modelllama                     & \modeldatajuicer                & \modelllama                      & \modeldatajuicer                & \modelllama \\
\midrule[1.0pt]
FedHDS$\ddagger$                        & 25.45 \std{0.64} \pdata{1.62}&  28.77 \std{1.91} \pdata{3.86}                    & 31.82 \std{0.56} \pdata{3.71} & 33.66 \std{0.48} \pdata{4.25}                    & 32.12 \std{1.66} \pdata{3.64} & 33.81 \std{0.87} \pdata{4.32} \\
FedHDS$\dagger$                 & 24.48 \std{0.78} \pdata{0.83} & 32.27 \std{0.12} \pdata{1.51}                  & 32.52 \std{0.65} \pdata{2.93} & 34.11 \std{0.94} \pdata{3.97}                        & 32.79 \std{1.15} \pdata{2.99} & 33.98 \std{1.36} \pdata{4.33} \\
\midrule[0.5pt]
\app                          & \texl{26.64 \std{0.79}} \pdata{0.20} & 32.32 \std{0.92} \pdata{0.15}                   & \textbf{33.38 \std{0.40}} \pdata{0.82} & \textbf{35.40 \std{0.78}} \pdata{1.18}                             & \textbf{33.70 \std{0.19}} \pdata{0.88} & \textbf{35.79 \std{0.43}} \pdata{1.34}  \\
\ \ \ \ w/i PCA   & \texb{27.84 \std{0.99}} \pdata{0.09} & \texl{32.62 \std{0.62}} \pdata{0.08}       & 33.12 \std{0.10} \pdata{0.74} & 34.90 \std{0.52} \pdata{1.86}           & 33.06 \std{0.23} \pdata{0.74} & 34.60 \std{1.08} \pdata{2.12}    \\
\ \ \ \ w/i KPCA  & 26.28 \std{3.93}  \pdata{0.10}& 28.95 \std{0.92} \pdata{0.08}  & 33.21 \std{0.77} \pdata{0.72} & 34.80 \std{0.95} \pdata{1.62}         & 33.34 \std{0.36} \pdata{0.75} & 35.11 \std{0.27} \pdata{2.08} \\
\midrule[0.5pt]
\apppro                     & 25.93 \std{0.75} \pdata{0.23} & \textbf{32.93 \std{0.64}} \pdata{0.22}           & \texl{33.28 \std{0.44}} \pdata{1.31} & 35.01 \std{0.65} \pdata{1.26}         & \texl{33.52 \std{0.20}} \pdata{1.19} &  \texl{35.42 \std{0.29}} \pdata{1.28} \\
\ \ \ \ w/i PCA  & 22.93 \std{2.96} \pdata{0.21} & 29.99 \std{0.78} \pdata{0.21}          & 33.03 \std{0.26} \pdata{1.40} & \texl{35.35 \std{0.65}} \pdata{1.87}           & 33.15 \std{0.21} \pdata{0.66} & 35.21 \std{0.54} \pdata{1.72} \\
\ \ \ \ w/i KPCA & 25.80 \std{2.78} \pdata{0.03} & 28.98 \std{2.33} \pdata{0.20}          & 32.86 \std{1.08} \pdata{1.86}   & 35.04 \std{0.73} \pdata{1.69}            & 33.40 \std{0.41} \pdata{0.71} & 34.88 \std{1.25} \pdata{1.92} \\
\bottomrule[1.0pt]
\end{tabular}
\end{adjustbox}
\end{table*}
\subsection{Comparison on Time Efficiency}
\label{subsec-exp-efficiency}
From Table \ref{tab-speedup}, by reducing the number of local training steps, coreset methods improve efficiency over FedIT.
\apppro achieves a higher speedup due to the fewer data samples required for better accuracy compared to other coreset FL baselines.
Particularly, on NI with \modeldatajuicer, \apppro achieves a speedup of 48.8$\times$ over FedIT.
Since \apppro is comparable to \app in accuracy, it may be more applicable.
We provide a breakdown of time consumption in Appendix \ref{subsec:appendix:exp-time-breakdown} to clearly present the time consumption of each step in our approach.

\subsection{Ablation Studies}
\label{subsec-exp-ablation}
\paragraph{Two-Layer Selection}
To clarify the contributions of intra- and inter-client selections, we build two methods:
1) FedHDS$\ddagger$: Removes inter-client selection and groups data by last-layer features.
2) FedHDS$\dagger$: Uses only intra-client selection, selecting one data sample closest to the centroid in each local group.
From Table \ref{tab-ablation}, intra- and inter-client selections vary in effectiveness across different scenarios.
\app outperforms FedHDS$\ddagger$ in all scenarios, showing the efficacy of the two selections.

To emphasize the necessity of hierarchical selection, we construct \emph{GlobalSelect}, which sends all fused features to the server for global clustering, and \emph{GlobalSelect-Turbo} that extracts features with GPT-2.
From Figure \ref{fig:ablation:global-two-layer}, global selection may cause poor accuracy, potentially because clustering on the entire large-scale dataset yields suboptimal results.
Thus, it is necessary to select data hierarchically.
\begin{figure}[t]
\centering
\subfigure[\datani]{
\includegraphics[width=0.46\linewidth]{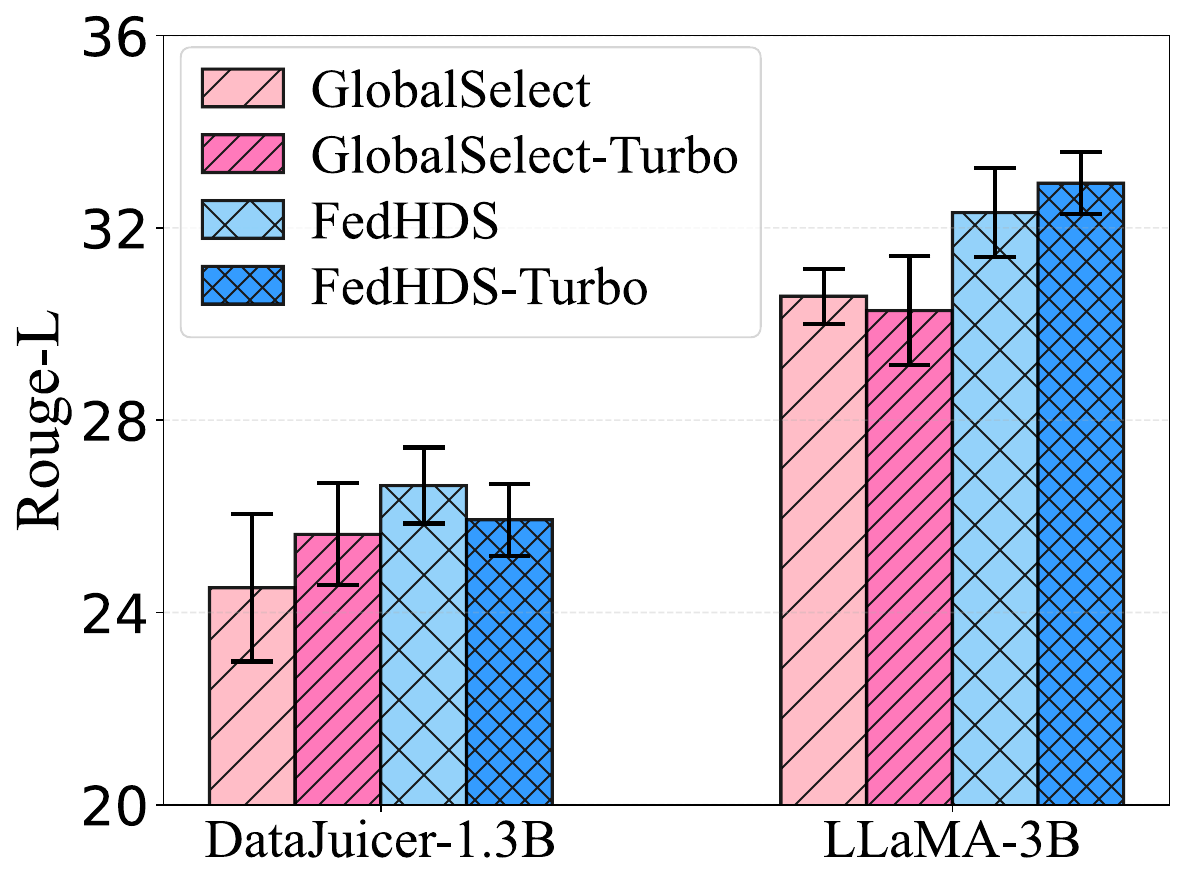}
}
\subfigure[\datadolly ($\alpha=0.5$)]{
\includegraphics[width=0.46\linewidth]{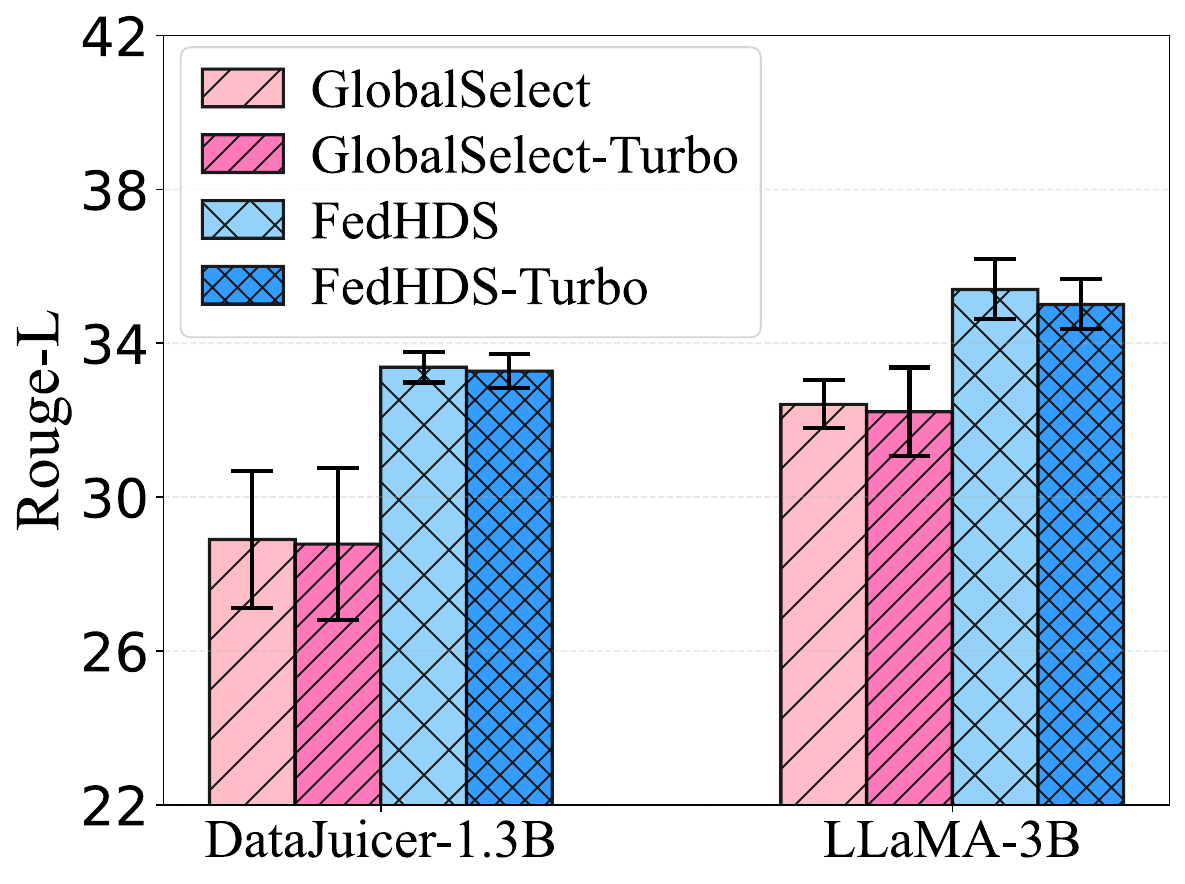}
}
\caption{
Performance of our approaches and methods that send data features to the server for global selection.
}
\label{fig:ablation:global-two-layer}
\end{figure}

\paragraph{Feature Fusion}
We explore the impact of the dimensionality reduction algorithm by replacing t-SNE with PCA \cite{jolliffe2002principal} and Kernel PCA \cite{scholkopf1997kernel}. 
From Table \ref{tab-ablation}, on the relatively simpler datasets (\datadolly), replacing t-SNE brings minimal differences. 
However, on NI, substituting t-SNE occasionally results in negative effects.
Thus, in practical scenarios, we recommend t-SNE for more effective coreset selection.

\subsection{Communication and Memory Costs}
\begin{table}[t]
\centering
\setlength\tabcolsep{3.6pt}
\renewcommand\arraystretch{0.88}
\caption{Per-round costs by 1.3B models (\datadolly).}
\label{tab:communication-memory}
\begin{adjustbox}{max width=\linewidth}
\begin{tabular}{l|c|c|c}
\toprule[1.0pt]
& \textbf{\makecell[c]{Comm.\\(Model)} } & \textbf{\makecell[c]{Comm. (Features\\\& Cluster Indices)}} & \textbf{GPU Mem.} \\
\midrule[1.0pt]
FedIT & 12 MB & 0 & 10.56 GB \\
FedHDS & 12 MB & 44 Bytes & ~9.40 GB \\
FedHDS-Turbo & 12 MB & 76 Bytes & ~9.32 GB \\
\bottomrule[1.0pt]
\end{tabular}
\end{adjustbox}
\end{table}
We provide the maximum memory cost and client-side per-round communication cost of our approaches.
Besides transmitting LoRA adapters as FedIT, i.e., Comm. (Model), \app additionally transmits 1) 2-dimensional cluster centroids and 2) indices of a few selected clusters.
From Table \ref{tab:communication-memory}, these bring negligible cost, i.e., just a few dozen bytes. 
Detailed calculations are in Appendix \ref{sec:appendix:calculation-communication}.

\app's memory usage is similar to FedIT, with a slight reduction from filtering long samples.
Nevertheless, \app makes computation offloading feasible for on-device LLM tuning by significantly reducing the required data samples, while training on the full dataset with enabling offloading incurs a substantial time cost (Figure \ref{fig:intro:time}).

\subsection{Studies on Convergence and Overfitting}
\label{subsec-exp-convergence}
\begin{figure*}[t]
    \centering
    \begin{minipage}[t]{0.56\linewidth}
        \centering
        \vspace{0pt}
        \includegraphics[width=\linewidth]{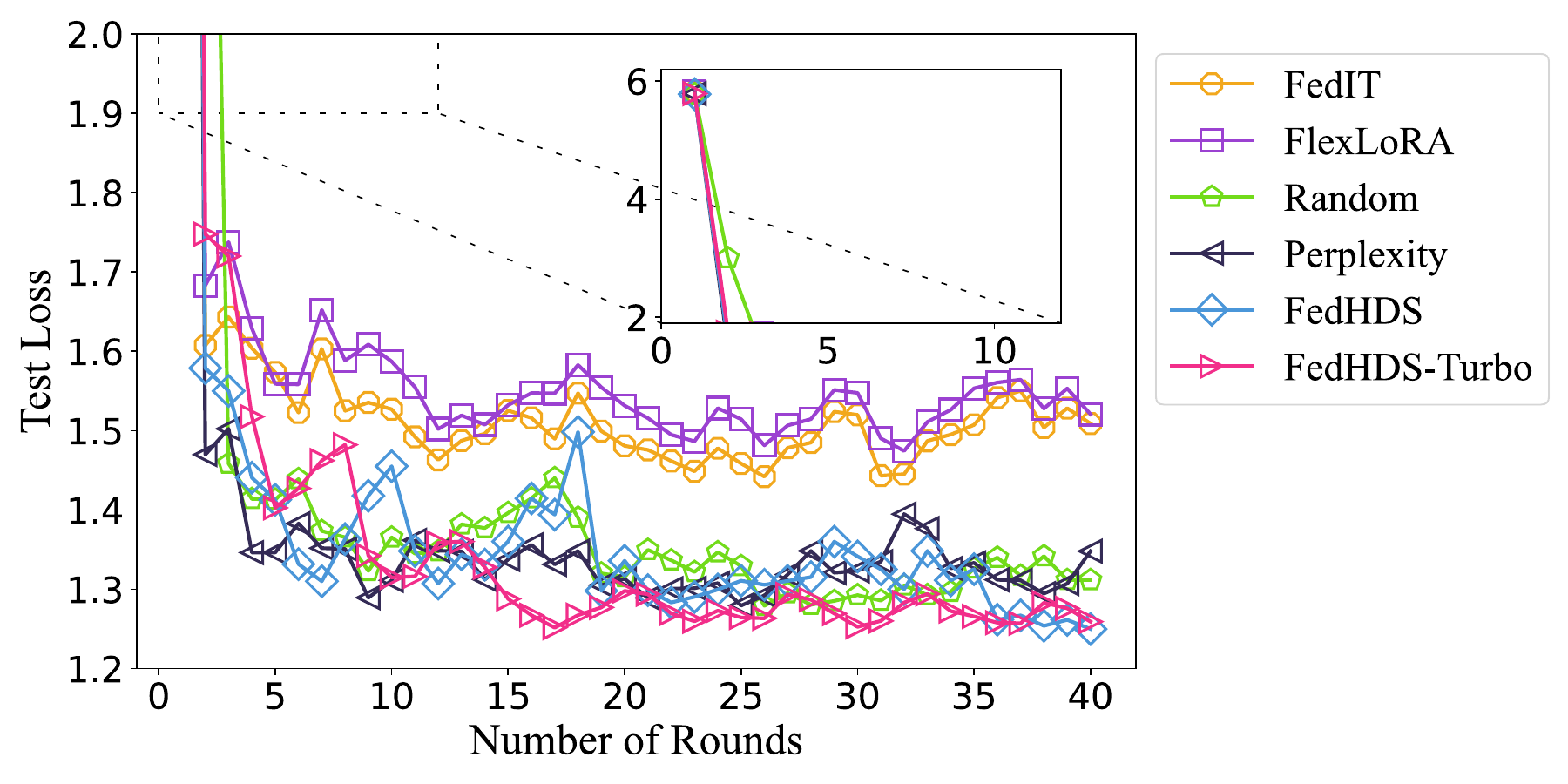}
        \caption{Convergence of the loss value on the test tasks obtained by \apppro and FedIT with \modelllama on NI.}
        \label{fig:convergence-3B-Instruct}
    \end{minipage}
    \hspace{0.3cm}
    \begin{minipage}[t]{0.338\linewidth}
        \centering
        \vspace{0pt}
        \includegraphics[width=\linewidth]{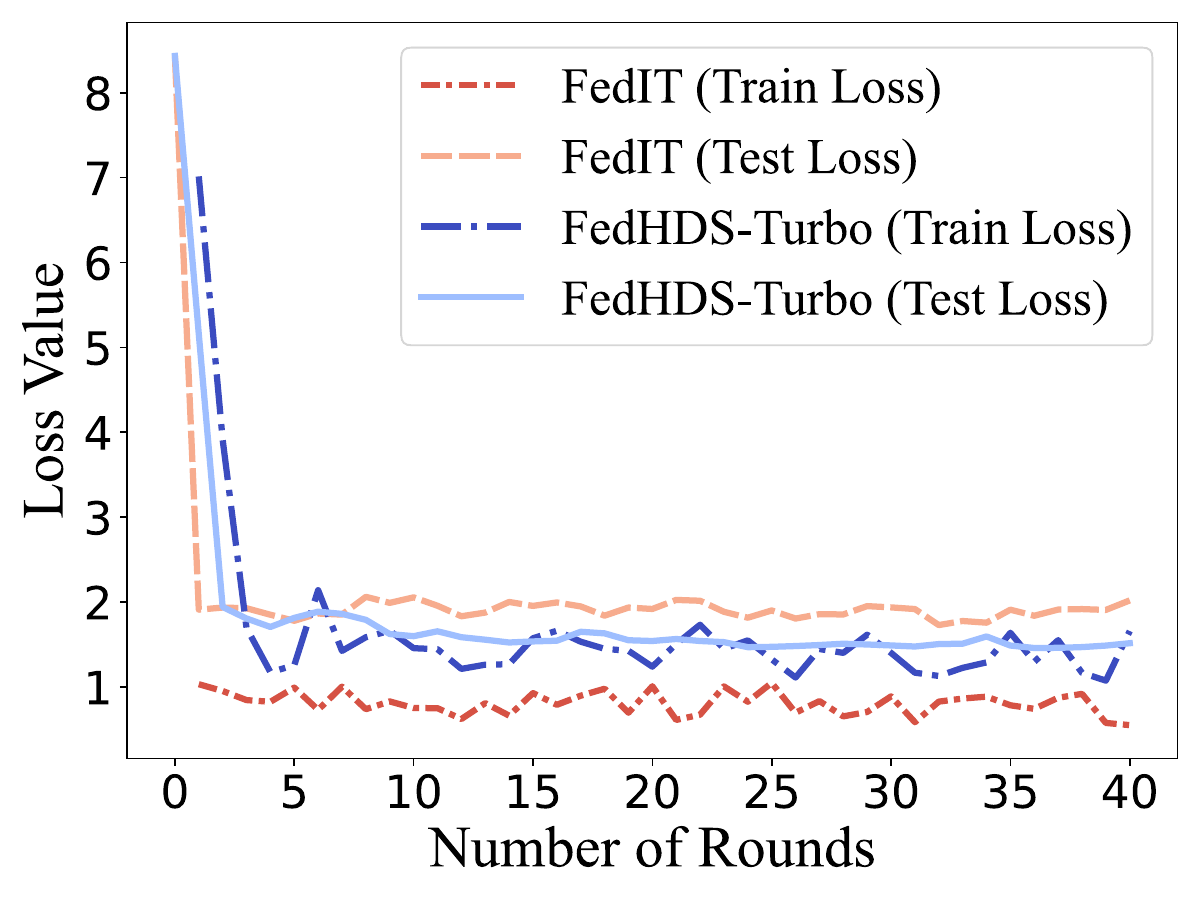}
        \caption{Convergence of training and test loss obtained by \apppro and FedIT with \modeldatajuicer on NI.}
        \label{fig:overfit-1B-Instruct}
    \end{minipage}
\end{figure*}
To illustrate the convergence trends of these approaches, Figure \ref{fig:convergence-3B-Instruct} presents the convergence curves of \app, \apppro, and the baseline methods using \modelllama on \datani.
The involved hyperparameters are aligned with those described in Section \ref{subsec:exp-setup}.
It can be seen that these methods have nearly reached a convergent state by the 40th federated round, demonstrating that LLMs can be effectively tuned with a limited number of instruction data.
Compared to approaches that use the full dataset, those relying on a data subset generally achieve a lower test loss. This is because tuning on the complete local dataset results in excessive local training steps, causing the LLM to overfit the local data and perform worse on unseen tasks, as discussed in Section \ref{subsec:exp-main}.

To better demonstrate that training on a subset can alleviate the overfitting problem to some extent, we provide the training and test loss of \apppro with \modeldatajuicer on \datani in Figure \ref{fig:overfit-1B-Instruct}, together with those of the federated approach using the full dataset that achieves the best average accuracy, i.e., FedIT.
As shown in Figure \ref{fig:overfit-1B-Instruct}, both the training and test loss of \apppro stably decrease. 
In contrast, the test loss of FedIT quickly stops decreasing while the training loss is still reducing to lower values. 
Finally, the test loss of FedHDS-Turbo is lower than that of FedIT, causing FedHDS-Turbo to perform better than FedIT, as shown in Table \ref{tab-performance}.

From the above, we can conclude that \app and \apppro exhibit stable convergence and perform on par with federated approaches that use the full dataset. 
Additionally, by reducing the amount of data used for instruction tuning, they mitigate overfitting to local data to some extent, leading to higher instruction-tuning accuracy compared to the baseline methods.

\subsection{Performance with Differential Privacy}
\label{subsec:exp:privacy}
\app can perform well with noise sampling variance no greater than 0.1, which improves the privacy of client-side data.
Limited by space, detailed experiments on this are left in Appendix \ref{subsec:appendix:exp-privacy}.

\subsection{Performance in Various FL Scenarios}
\label{subsec:exp:performance-scenrios}
\app outperforms coreset FL baselines with different active client ratios, showing its applicability in various FL scenarios.
We demonstrate this with experimental results in Appendix \ref{subsec:appendix:exp-performance-scenrios}.
\section{Conclusion}
\label{sec-conclusion}
Existing federated instruction tuning methods for LLMs typically train LLMs using all local data, causing significant computation cost and overfitting to local data. 
This work pioneers an exploration into federated data-efficient instruction tuning, and proposes \app, a coreset selection approach that solves both intra- and inter-client data redundancy.
It fuses data features of varying abstraction levels obtained from different Transformer layers for better data representation.
Extensive experiments involving various datasets, LLMs and non-IIDness demonstrate that \app enhances the data efficiency and Rouge-L on unseen tasks over existing federated tuning methods.

\section{Limitations}
\label{sec:limitation}
Although our approach improves data efficiency and generalization to unseen tasks in FL for LLM fine-tuning, it still has certain limitations.
For example, it only selects data based on representativeness but overlooks data quality.
Since domain divisions are usually implicit, low-quality data samples may also be treated as a separate domain. 
In this case, \app may select low-quality data. 
Therefore, incorporating quality-based filtering mechanisms may help further improve our approach.

\bibliography{references}

\appendix

\clearpage
\addcontentsline{toc}{section}{Appendix}
\part{\large{Appendix}}
We provide more discussions and experiments of this work and organize them as follows:

\parttoc
\section{Detailed Algorithm}
\label{sec:appendix:algo}
To facilitate a better understanding of each step of the proposed approach, we provide Algorithm \ref{algo:data-selection} to explain how \app selects the coresets for activate clients in each round $r$, where lines \ref{line:loop-intra}\textasciitilde\ref{line:submit-centroids} and \ref{line:loop-coreset}\textasciitilde\ref{line:ret} are performed individually by each client, and lines \ref{line:initialize-centroids} and \ref{line:call-inter}\textasciitilde\ref{line:send-selected} are performed by the server.
\begin{algorithm*}
\SetAlgoLined
\DontPrintSemicolon
\SetNoFillComment
\SetKwFunction{FClientTraining}{ClientTraining}
\SetKwFunction{FIntraSelection}{IntraClientSelection}
\SetKwFunction{FInterSelection}{InterClientSelection}
\SetKwProg{Fn}{Function}{:}{end} 
\setlength{\abovedisplayskip}{3pt}
\setlength{\belowdisplayskip}{3pt}
\setlength{\abovedisplayshortskip}{3pt}
\setlength{\belowdisplayshortskip}{3pt}
\caption{Processes of Data Selection in \textbf{\app} in each round $r$ of FL.}
\label{algo:data-selection}
\KwIn{$\mathbb{M}_r$, $\left\{\mathcal{D}_1, \dots, \mathcal{D}_{\left|\mathbb{M}_r \right|}\right\}$.}
\KwOut{The selected coreset $\widetilde{\mathcal{D}}_i$ for each client $i$ active in this round, denoted as $\left\{\widetilde{\mathcal{D}}_1, \dots, \widetilde{\mathcal{D}}_{\left|\mathbb{M}_r \right|}\right\}$.}
\vspace{0.2cm}
Initialize a list $\mathbb{C}$ to stage the received centroids \label{line:initialize-centroids}\\
\For{$i = 1, 2, \ldots, \left|\mathbb{M}_r\right|$}{
 \label{line:loop-intra}
 $\left\{\mathbf{c}_1, \mathbf{c}_2, \ldots \right\}$ = \FIntraSelection{$\mathbf{w}$, $\mathcal{D}_i$} \label{line:call-intra}
 \\
 $\mathbb{C} = \mathbb{C} \cup \left\{\mathbf{c}_1, \mathbf{c}_2, \ldots \right\}$ \label{line:submit-centroids} \tcp*{\ding{195} in Figure \ref{pic-framework}}
}
$\mathbb{G}^{\text{selected}}$ = \FInterSelection{$\mathbb{C}$} \label{line:call-inter}\\ 
send indices of selected groups to corresponding clients \label{line:send-selected}\tcp*{\ding{197} in Figure \ref{pic-framework}}
\For{$i = 1, 2, \ldots, \left|\mathbb{M}_r\right|$}{
 \label{line:loop-coreset}
 $\widetilde{\mathcal{D}}_i = \left\{\mathbf{x} \mid q = \underset{\mathbf{x} \in \mathcal{G}_j}{\operatorname{arg \ min}} \left\| \mathbf{x} - \mathbf{c}_j\right\| \right\}^{\mathcal{G}_j \in \mathbb{G}_i, \mathcal{G}_j \in \mathbb{G}^{\text{selected}}}$  \label{line:inloop-coreset} \tcp*{\ding{198} in Figure \ref{pic-framework}}
}
\KwRet{
\label{line:ret}
the selected subset $\widetilde{\mathcal{D}}_i$ for each client $i \in \mathbb{M}_r$.}

\vspace{0.3cm}
\Fn{\FIntraSelection{$\mathbf{w}$, $\mathcal{D}$}}{
\For{$j = 1, 2, \ldots, |\mathcal{D}|$}{
  extract features of $\mathbf{x}_j$ by each layer of $\mathbf{w}$, denoted by $\mathbf{h}_j = \left [\mathbf{h}^{1, -1}_{j}, \mathbf{h}^{2, -1}_{j}, \ldots, \mathbf{h}^{l, -1}_{j}  \right ]$  \tcp*{\ding{192} in Figure \ref{pic-framework}}
}
reduce dimensionality of $\mathbb{H} = \left\{\mathbf{h}_1, \mathbf{h}_2, \ldots, \mathbf{h}_{|\mathcal{D}|}\right\}$ as $\left\{\widetilde{\mathbf{h}}_1, \widetilde{\mathbf{h}}_2, \ldots, \widetilde{\mathbf{h}}_{\left|\mathcal{D}_i\right|} \right\} = P(\left\{\mathbf{h}_1, \mathbf{h}_2, \ldots, \mathbf{h}_{\left|\mathcal{D}_i\right|} \right\})$, obtaining the fused features  \tcp*{\ding{193} in Figure \ref{pic-framework}}
cluster data in $\mathcal{D}_i$ based on their fused features, as $\left\{\mathcal{G}_1, \mathcal{G}_2, \ldots \right\} = \operatorname{HDBSCAN}(\left\{\widetilde{\mathbf{h}}_1, \widetilde{\mathbf{h}}_2, \ldots, \widetilde{\mathbf{h}}_{\left|\mathcal{D}_i\right|} \right\})$, where each $\mathcal{G}_j$ corresponds to an approximate centroid $\mathbf{c}_j$  \tcp*{\ding{194} in Figure \ref{pic-framework}}
\KwRet{$\left\{\mathbf{c}_1, \mathbf{c}_2, \ldots \right\}$}
}

\vspace{0.3cm}
\Fn{\FInterSelection{$\left\{\mathbf{c}_1, \mathbf{c}_2, \ldots \right\}$}}{
perform HDBSCAN to partition received $\left\{\mathbf{c}_1, \mathbf{c}_2, \ldots \right\}$ into several groups $\left\{\mathcal{G}^{\text{II}}_1, \mathcal{G}^{\text{II}}_2, \ldots\right\}$  \tcp*{\ding{196} in Figure \ref{pic-framework}}
initialize a list $\mathbb{G}^{\text{selected}}$ \\
\For{$j = 1, 2, \ldots, \left|\left\{\mathcal{G}^{\text{II}}_1, \mathcal{G}^{\text{II}}_2, \ldots\right\}\right|$}{
   $s = \underset{\mathbf{c} \in \mathcal{G}^{\text{II}}_j}{\operatorname{arg \ min}} \left \| \mathbf{c} - \mathbf{c}^{\text{II}}_j \right \| $ \\
   add $\mathcal{G}_s$ into $\mathbb{G}^{\text{selected}}$ \\
}
}
\KwRet{$\mathbb{G}^{\text{selected}}$}
\end{algorithm*}
\section{Convergence Analysis}
\label{sec:appendix:convergence-analysis}
\app shares the same global objective with FedAvg (FedIT) defined in Eq. \eqref{eq-fl-optimization}.
The essential difference between \app and FedAvg (FedIT) is that it optimizes only on a subset of the original data, as illustrated by the difference between Eq. \eqref{eq-local-sgd-subset} and Eq. \eqref{eq-local-sgd}.
Given the global objective of \app defined in Eq. \eqref{eq-fl-optimization}, we have the local objective of each client $i$ during local training as:
\begin{equation}
    \min_{\mathbf{w}} f_i(\mathbf{w}) \triangleq \frac{1}{\left|\widetilde{\mathcal{D}}_i\right|} \sum_{i=1}^{\left|\widetilde{\mathcal{D}}_i\right|} \mathbb{E}_{\mathbf{x} \sim \widetilde{\mathcal{D}}_i}\left[ \mathcal{L}(\mathbf{w}; \mathbf{x})\right].
\end{equation}
Following \citet{li2020convergence}, we first make the following assumptions:
\begin{assumption}
    The local objective of each client $i$ is $L$-smooth, i.e., $f_i(\mathbf{v}) \leq f_i (\mathbf{w}) + (\mathbf{v} - \mathbf{w})^T \nabla f_i(\mathbf{w}) + \frac{L}{2}\left\|\mathbf{v} - \mathbf{w} \right \|_2^2$, $\forall \mathbf{v} \in \mathbb{R}^d$, $\forall\mathbf{w} \in \mathbb{R}^d$.
    \label{assumtion:l-smooth}
\end{assumption}
\begin{assumption}
    The local objective of each client $i$ is $\mu$-convex, i.e., $f_i(\mathbf{v}) \geq f_i (\mathbf{w}) + (\mathbf{v} - \mathbf{w})^T \nabla f_i(\mathbf{w}) + \frac{\mu}{2}\left\|\mathbf{v} - \mathbf{w} \right \|_2^2$, $\forall \mathbf{v} \in \mathbb{R}^d$, $\forall\mathbf{w} \in \mathbb{R}^d$.
    \label{assumtion:mu-convex}
\end{assumption}
\begin{assumption}
    The variance of the stochastic gradients across each client is bounded, i.e., $\mathbb{E} \left\|\nabla f_i(\mathbf{w}_{i,\tau}, \mathbf{x})  - \nabla f_i(\mathbf{w}_{i,\tau})\right\|\leq \sigma_i^2$, where $\mathbf{w}_{i,\tau}$ denotes the model parameters of the $i$-th client after $\tau$ steps of updates.
    \label{assumtion:gradient-device-bound}
\end{assumption}
\begin{assumption}
    The expected squared norm of the stochastic gradients stays within a uniform bound, i.e., $\mathbb{E} \left\|\nabla f_i(\mathbf{w}_{i,\tau}, \mathbf{x}) \right\|^2 \leq G^2$.
    \label{assumtion:gradient-uniform-bound}
\end{assumption}
Existing works also make similar assumptions on the convergence analysis of federated tuning of LLMs \cite{ling2024convergence}. 
We also mildly assume that in each round, there are $K$ clients on average that will submit their tuned models to the server, and each of them performs $E$ steps of local training to update their local models.
\begin{theorem}
    Let Assumptions \ref{assumtion:l-smooth}\textasciitilde\ref{assumtion:gradient-uniform-bound} hold, $\mathbf{w}^*$ be the optimal global model, $\kappa=\frac{L}{\mu}$, $\gamma = \max\left\{8\kappa, E\right\}$, $B = \sum_{i=1}^{N} \lambda_i^2 \sigma_i^2 + 6LT + 8(E-1)^2G^2$, $C = \frac{4}{K} E^2G^2$, and $\mathbf{E} = \mathbb{E}\left[f(\mathbf{w}_T)\right] - f(\mathbf{w}^*)$, after $T$ iterations, we have
    \begin{equation}
         \mathbf{E} \leq \frac{2\kappa}{\gamma + T}\left(\frac{B+C}{\mu} + 2L \left\| \mathbf{w}^0 - \mathbf{w}^*\right\|^2\right).
    \end{equation}
    \label{theorem:convergence:fedhds}
\end{theorem}

\begin{proof}
    In fact, \app is based on FedAvg and performs downsampling on local data, thereby affecting the number of local training iterations. 
    With appropriate variable substitution, the convergence of FedHDS can be derived from the proof process in the work of \citet{li2020convergence}.
\end{proof}
Based on Theorem \ref{theorem:convergence:fedhds}, \app has a convergence guarantee. 
Compared to FedIT, which directly adopts the training processes of FedAvg, the convergence rate of \app on the training set theoretically has certain disadvantages, which have been experimentally demonstrated in Figure \ref{fig:convergence-3B-Instruct}. 
However, the advantage of \app is its ability to handle overfitting to local data.
As shown in both Figures \ref{fig:convergence-3B-Instruct} and \ref{fig:overfit-1B-Instruct}, \app and \apppro significantly outperform FedIT in terms of test loss. 
Therefore, \app and \apppro achieve better Rouge-L on held-out tasks than FedIT, as presented in Table \ref{tab-performance}. 

\section{Analysis on Speedup Ratio}
\label{sec:appendix:speed-up}
In this section, we provide a brief analysis of the speedup achievable by our approaches to better understand their effectiveness in acceleration based on numerical results in Table \ref{tab-speedup}.

For a client with $M$ data samples, the time complexity of performing t-SNE with Barnes-Hut implementation is $\mathcal{O}(M\log M)$, and that of performing HDBSCAN is $\mathcal{O}(M\log M) \sim \mathcal{O}(M^2)$ based on data sparsity (we adopt the worst complexity). 
Assume that $\nu$ and $\epsilon$ are the scale constants between the actual time consumption and complexity of t-SNE and HDBSCAN, respectively, and $\xi$ and $\Xi$ denote the time consumption (e.g., seconds) of performing one-step inference and training with one data sample, respectively.
Assuming that the proportion of training data that \app can filter out is $\varsigma$, the time consumption of a client by conducting LLM instruction tuning with \app is 
\begin{equation}
    \nu\cdot M\log M + \epsilon\cdot M^2 + \xi \cdot M + \Xi(1-\varsigma)M,
\end{equation}
while that of FedIT, the approach using full data, is $\Xi\cdot M$.
Therefore, the speeding-up ratio can be formalized as:

\begin{equation}
    \frac{\Xi\cdot M}{\nu\cdot M\log M + \epsilon\cdot M^2 + \xi \cdot M + \Xi(1-\varsigma)M}.
\end{equation}

For these notations:

\begin{itemize}
    \item Generally, $\Xi > \xi$ by several times. 
    \item From the experiments on Dolly-15K, t-SNE on 200 samples takes 0.5 seconds, and HDBSCAN takes only 0.003 seconds. Based on these statistics, $\nu$ could be in the order of $10^{-4}$, and $\epsilon$ could be in the order of $10^{-8}$. Therefore, $\xi \gg \nu$ and $\xi \gg \epsilon$.
    \item $\varsigma$ can exceed 99\% on \datani.
\end{itemize}

Therefore, the speedup achieved by our approaches is significant, as experimentally demonstrated in Table \ref{tab-speedup}.
Besides, the inference can be significantly accelerated (by reducing $\xi$), leading to a considerable improvement, e.g., \apppro is significantly faster than \app.
\section{Reproducibility}
\label{sec:appendix:reproducibility}

\subsection{Experimental Environments for Accuracy Evaluation}
\label{subsec:appendix:reproducibility-exp-env}
We implement these approaches mentioned above with PyTorch \cite{paszke2019pytorch} \texttt{2.0.1}, Transformers \cite{wolf-etal-2020-transformers} \texttt{4.31.0}, scikit-learn \texttt{1.5.1}, PEFT \cite{peft} \texttt{0.4.0}, and hdbscan \texttt{0.8.37}.
Numerical experiments in Tables \ref{tab-performance} and \ref{tab-ablation} are performed on platforms equipped with NVIDIA A100 or NVIDIA A6000 GPUs, installed with Python \texttt{3.10} and CUDA \texttt{12.4}.
Efficiency results in Table \ref{tab-speedup} are obtained on a platform with an NVIDIA A6000 GPU, installed with Python \texttt{3.10}, CUDA \texttt{12.4} and DeepSpeed \texttt{0.15.2}.

\subsection{Experimental Environments for Memory Footprint and Efficiency Statistics}
\label{subsec:appendix:mem-time}
The memory footprint and time consumption in Figures \ref{fig:intro:mem} and \ref{fig:intro:time} are measured with a maximum token list length of 1,024 where excessively long data will be truncated.
The adopted platform is equipped with an NVIDIA A6000 GPU, installed with Python \texttt{3.10}, CUDA \texttt{12.4} and DeepSpeed \texttt{0.15.2}.
For the memory footprint, the 95th percentile is calculated.
The selected two GPUs are based on the most popular desktop and laptop GPUs identified from the Steam Hardware Survey (Dec 2024).

\subsection{Detailed Hyperparameters}
\label{subsec:appendix:reproducibility-hyperparameters}
In approaches involving the LoRA adapter, i.e., \app, \apppro, FedIT and FlexLoRA, the adapters are configured with the same hyperparameter settings, i.e., \texttt{rank}, \texttt{alpha} and \texttt{dropout} of LoRA adapters are set to 8, 16 and 0.05 for \app and \apppro, respectively.
Note that although FlexLoRA supports heterogeneous-rank LoRA adapters, we adopt a homogeneous-rank setting to ensure a fair comparison.
Coreset-based methods perform 60 rounds of FL on \datadolly with \modelllama, and 40 rounds for other scenarios, where the local training is performed on the coreset for one epoch with Adam optimizer.
The learning rate and number of rounds for federated approaches using full data are aligned with those in the work of \citet{qin2024full}.

We perform a preliminary hyperparameter search for federated approaches using coresets, and adopt the advantageous settings for each approach in each scenario for large-scale experiments.
Specifically, we first search the learning rate in $\left\{3\times 10^{-4}, 1\times 10^{-4}, 3\times 10^{-5}\right\}$.
Then, for Random and Perplexity, we search the ratio of data samples in the final selected subsets to the full data samples in $\left\{0.2\%, 1.5\%, 2\%, 5\%\right\}$.
Note that the thresholds of 0.2\% and 1.5\% are inspired by the proportion automatically obtained by \app and \apppro. 
Considering the importance of an appropriate data ratio \cite{cao2024instruction}, these two baselines have benefited to some extent from the data proportion provided by \app and \apppro.
The finally adopted values for Tables \ref{tab-performance} and \ref{tab-ablation} as as follows:
All the federated approaches with coresets adopt a learning rate $\eta$ of $3\times 10^{-5}$ on \datadolly.
On \datani, \app and \apppro adopt $\eta=3\times 10^{-4}$, Random adopt $\eta=3\times 10^{-4}$ with \modeldatajuicer and $\eta=3\times 10^{-5}$ with \modelllama, Perplexity adopt $\eta=3\times 10^{-5}$ with \modeldatajuicer and $\eta=3\times 10^{-4}$ with \modelllama.

For Random and Perplexity, we searched for the optimal data ratio within $\left\{0.2\%, 1.5\%, 2\%, 5\%\right\}$, where 0.2\% and 1.5\% are inspired by those automatically obtained by \app and \apppro.
Our approaches adopt learning rate $\eta=3 \times 10^{-5}$ on \datadolly and $\eta=3 \times 10^{-4}$ on NI.
\app and \apppro apply HDBSCAN separately in both intra-client and inter-client selection.
During intra-client selection, the minimum cluster of HDBSCAN is set to the default value, i.e., 5 for \datani and 2 for \datadolly, considering the relatively small scale of \datadolly.
During inter-client selection, the minimum cluster of HDBSCAN is uniformly set to 2.

\subsection{Detailed Descriptions on Datasets}
\label{subsec:appendix:reproducibility-data-preprocessing}
This work adopts the same data preprocessing as reported in \cite{qin2024full}.
Following \citet{zhang2024fedit,qin2024full}, we adopt the prompt template from Alpaca \cite{alpaca}.

\datani includes a large collection of tasks with natural language instructions.
We adopt the dataset versioned by \texttt{v2.8} and the \texttt{default} split, which includes 756 tasks for training and 119 tasks for testing, each with a distinct task definition. 
Considering the scale of the dataset, experiments on it are conducted on a randomly sampled subset, containing 20\% of the data instances for each training task and 2\% for each test task. 
After the subsampling, each training task with at least 20 data instances is assigned to a unique client. 
After the above preprocessing, a federated scenario with 738 clients is formed, where the test tasks remain on the server for the held-out evaluation of the tuned LLMs.

\datadolly contains 15,015 data samples corresponding to 8 tasks, with the 1,188 data samples from the last task used for testing. 
The data from the remaining seven tasks is distributed to 200 clients for training, with each task labeled according to its respective task and distributed according to a Dirichlet distribution.
We partition data samples to clients via Dirichlet distribution based on the \texttt{category} attribute of each data sample.

After the last FL round is finished, Rouge-L scores are evaluated using the tuned LLM after the final FL round.
These scores are calculated based on the data samples in evaluation tasks, with the responses as the ground-truth labels.
\section{Additional Experiments}
\label{sec:appendix:additional-exp}

\subsection{Evaluation of Features from Different Transformer Layers}
\label{sec:appendix:additional-exp-metrics-layer}
\begin{figure}[t]
  \centering
  \subfigure[CH Index (1.3B) $\uparrow$]{
    \includegraphics[width=0.42\linewidth]{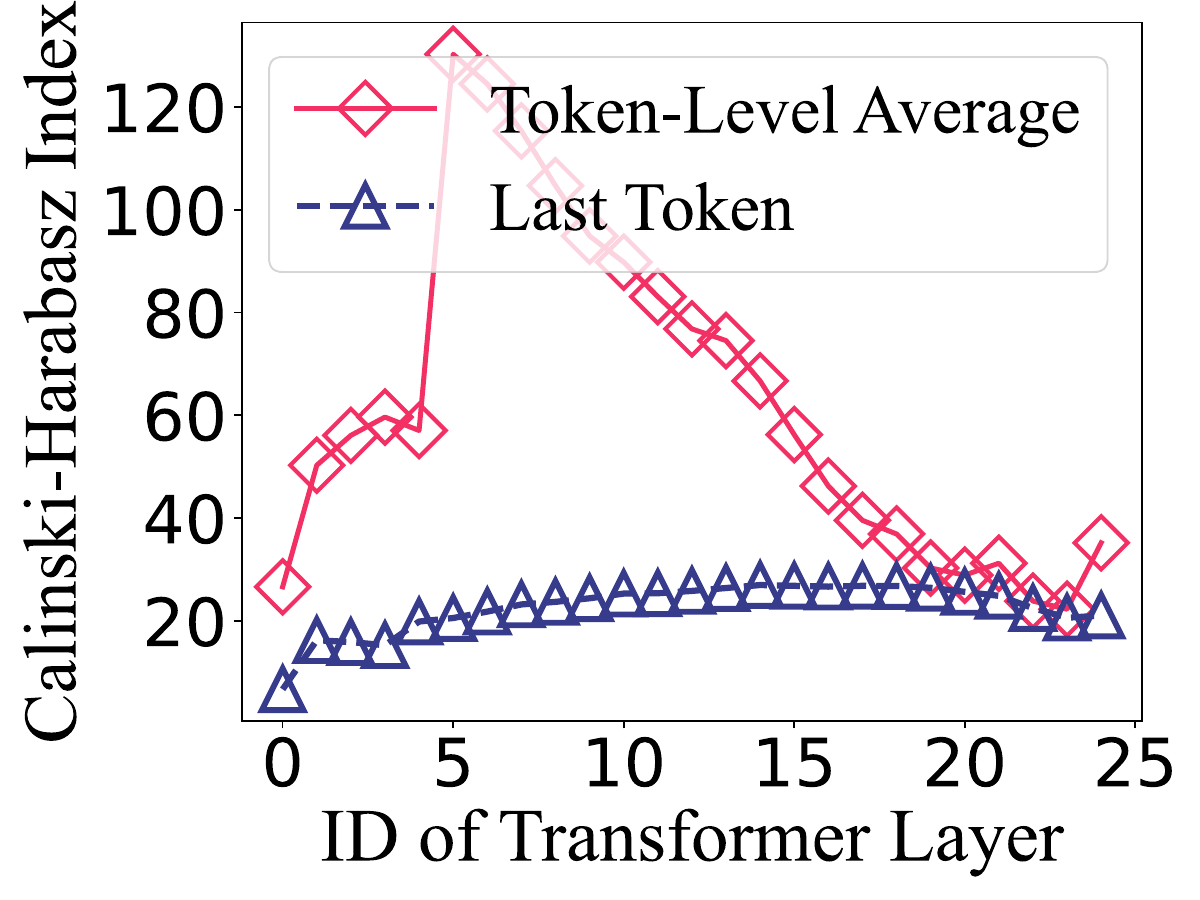}
  }
  \subfigure[DB Score (1.3B) $\downarrow$]{
    \includegraphics[width=0.42\linewidth]{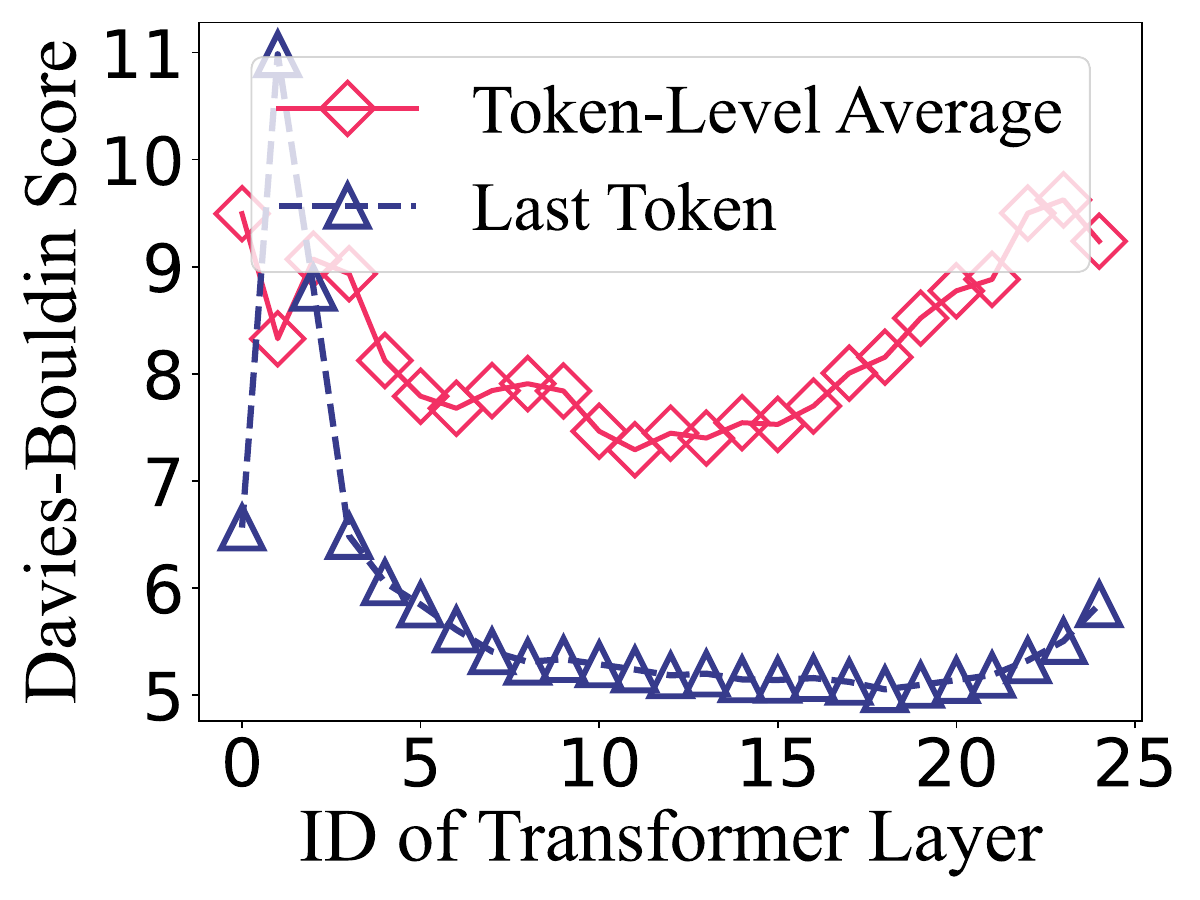}
  }
  \subfigure[SC Index (1.3B) $\uparrow$]{
    \includegraphics[width=0.42\linewidth]{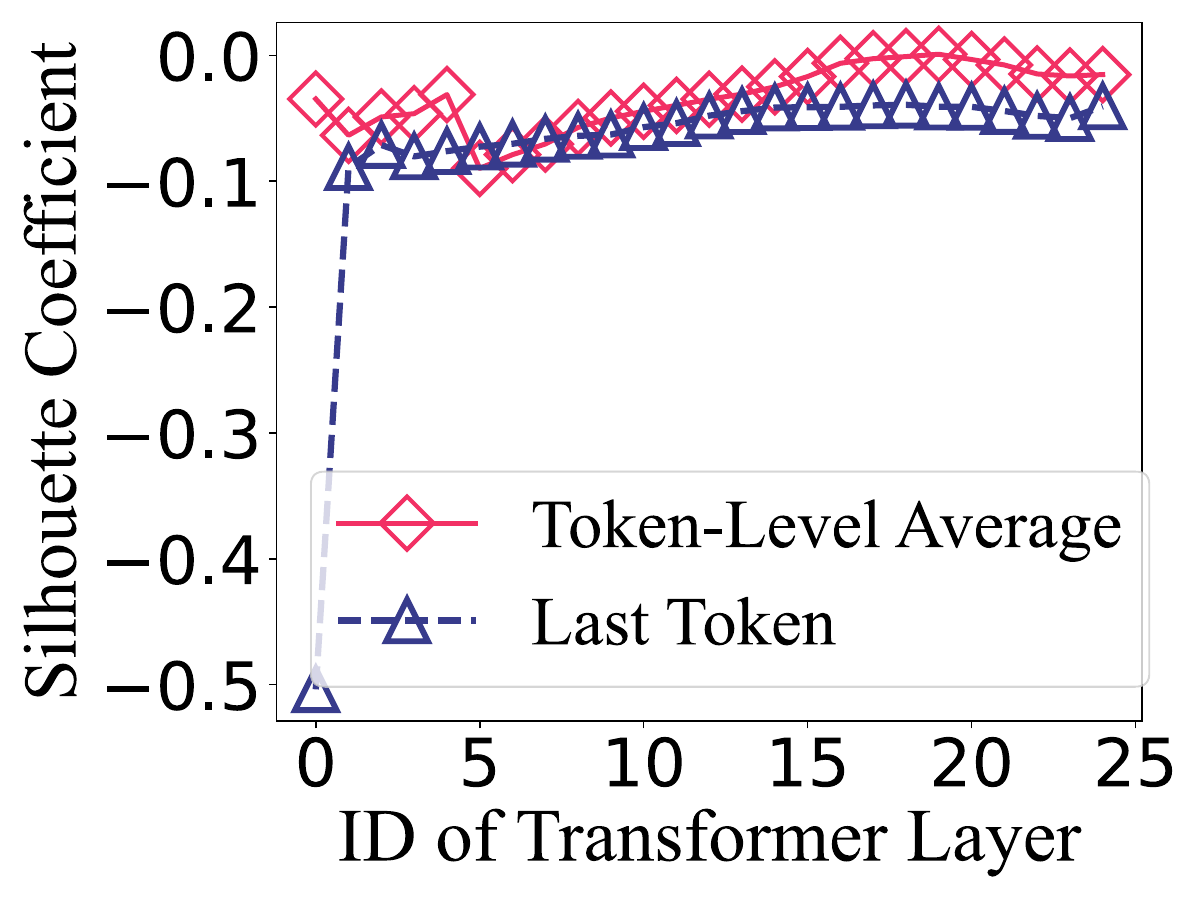}
  }
  \subfigure[F$_{1}$-Score (1.3B) $\uparrow$]{
    \includegraphics[width=0.42\linewidth]{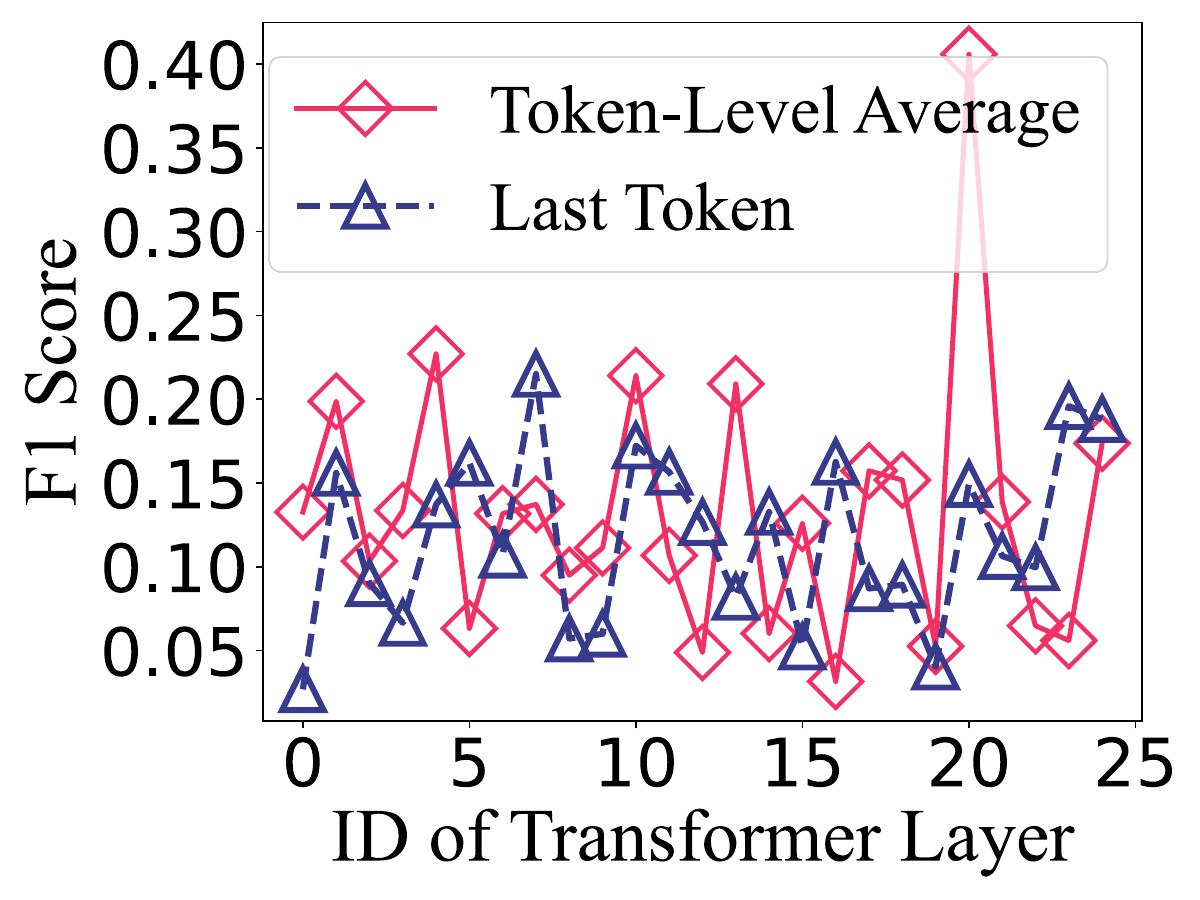}
  }
  \subfigure[DB Score (3B) $\downarrow$]{
    \includegraphics[width=0.42\linewidth]{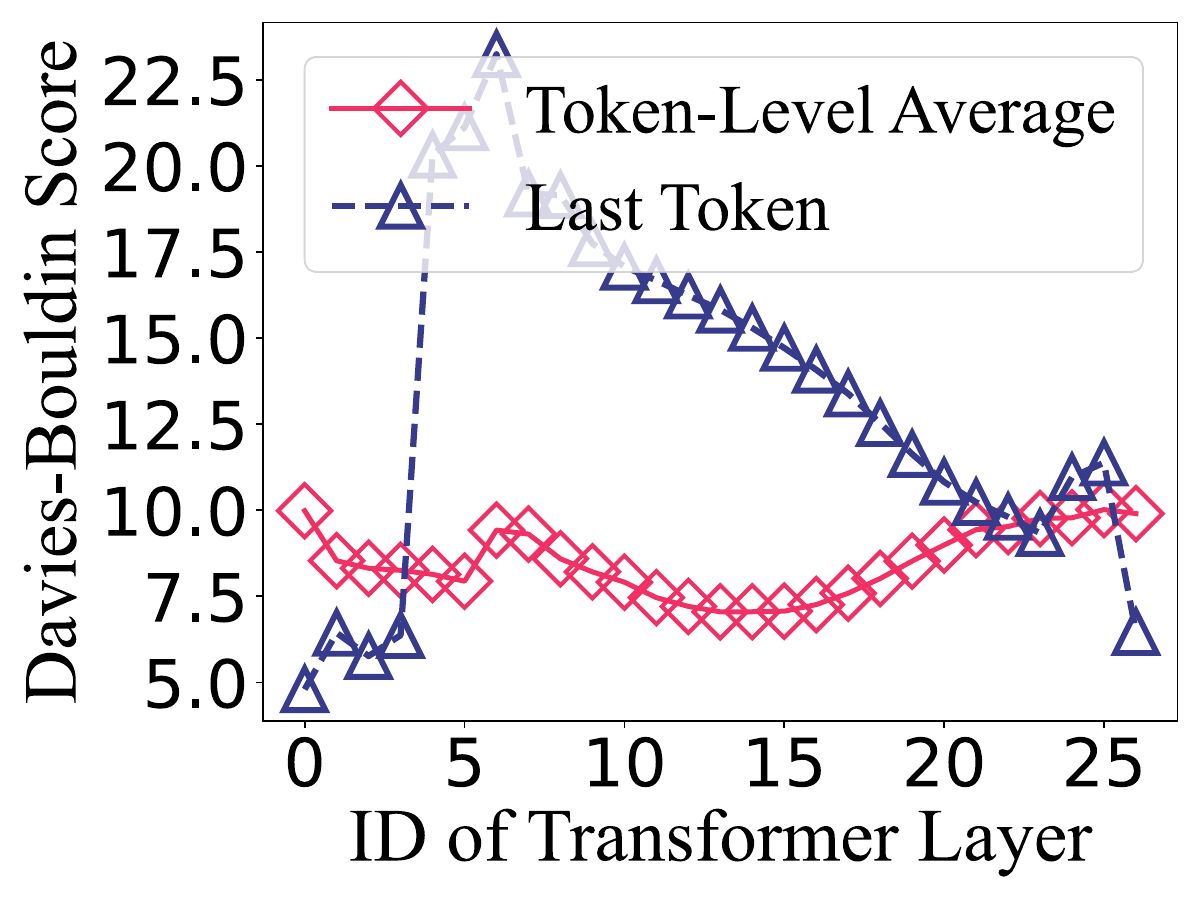}
  }
  \subfigure[SC Index (3B) $\uparrow$]{
    \includegraphics[width=0.42\linewidth]{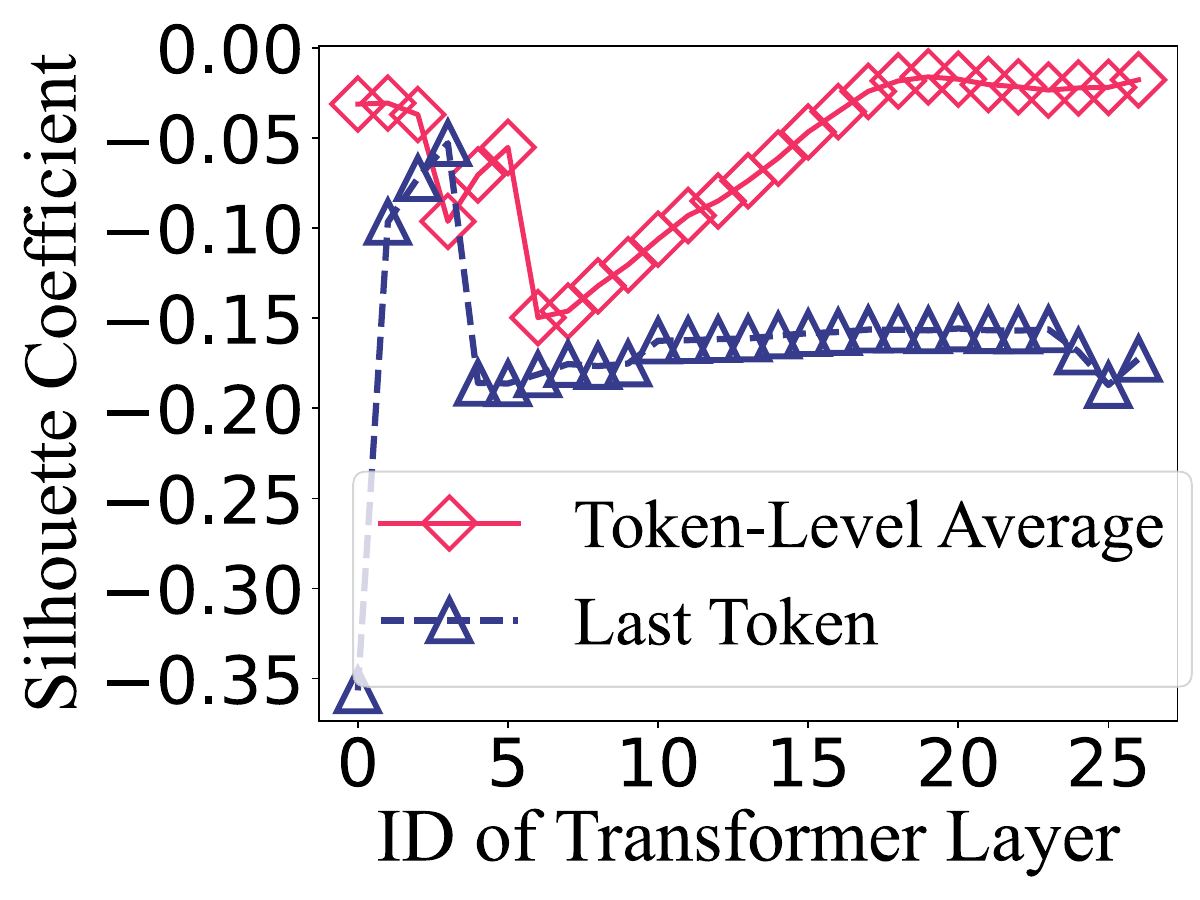}
  }
  \caption{Evaluations of clustered data groups based on features from different Transformer layers, obtained in a centralized scenario with different LLMs on \datadolly.
  }
  \label{fig:appendix:clustering-metrics-layer}
\end{figure}
To better demonstrate the statement that the last layer is not universally optimal, and no single layer excels across all metrics in Section \ref{subsec-app-intra}, we conduct evaluations with more LLMs and more metrics including Silhouette Coefficient \cite{rousseeuw1987silhouettes} and Davies-Bouldin score \cite{4766909}, following the experimental setups aligned to that of Figure \ref{pic-clustering-metrics-layer}.
The additional experimental results in Figure \ref{fig:appendix:clustering-metrics-layer} further demonstrate the above statement.

\subsection{Virtualization of Data Features}
\label{subsec:appendix:virtualization-data-features}
In Section \ref{subsec-app-inter}, we visualize the features obtained from the last Transformer layer and the fused features with \modeldatajuicer on \datadolly as an example, as presented in Figure \ref{pic-visualization-feature}.
In order to illustrate the differences between fused features and the final layer features in more scenarios, we supplemented the visualizations using different LLMs in various scenarios.
As shown in Figure \ref{pic-appendix-visualization-feature-aop}, the fused features create more distinct boundaries among clients. 
Thus, these fused features can be utilized to more effectively distinguish between data samples than existing solutions.
\begin{figure}[t]
  \centering
  \subfigure[Features by the last Transformer layer with \modelllama in \datadolly ($\alpha$=0.5).]{
    \includegraphics[width=0.45\linewidth]{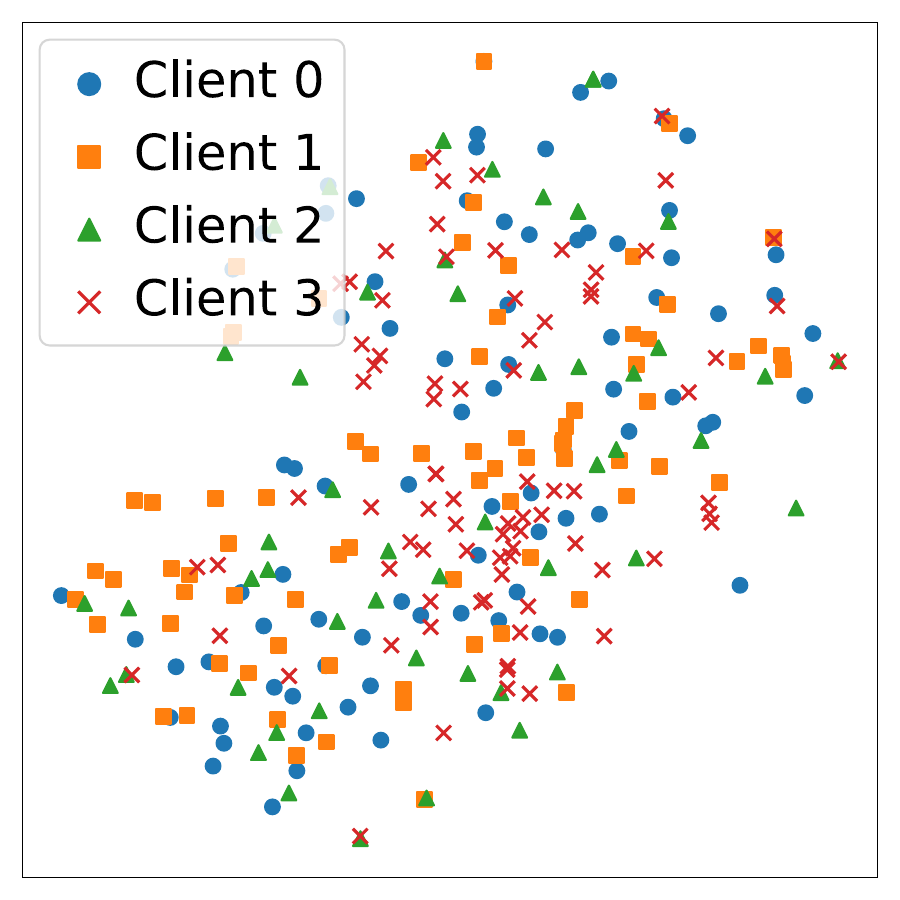}
    \label{subfig-vis-3B-dolly0.5-last}
  }
  \hspace{0.1cm}
  \subfigure[Fused features by \app with \modelllama in \datadolly ($\alpha$=0.5).]{
    \includegraphics[width=0.45\linewidth]{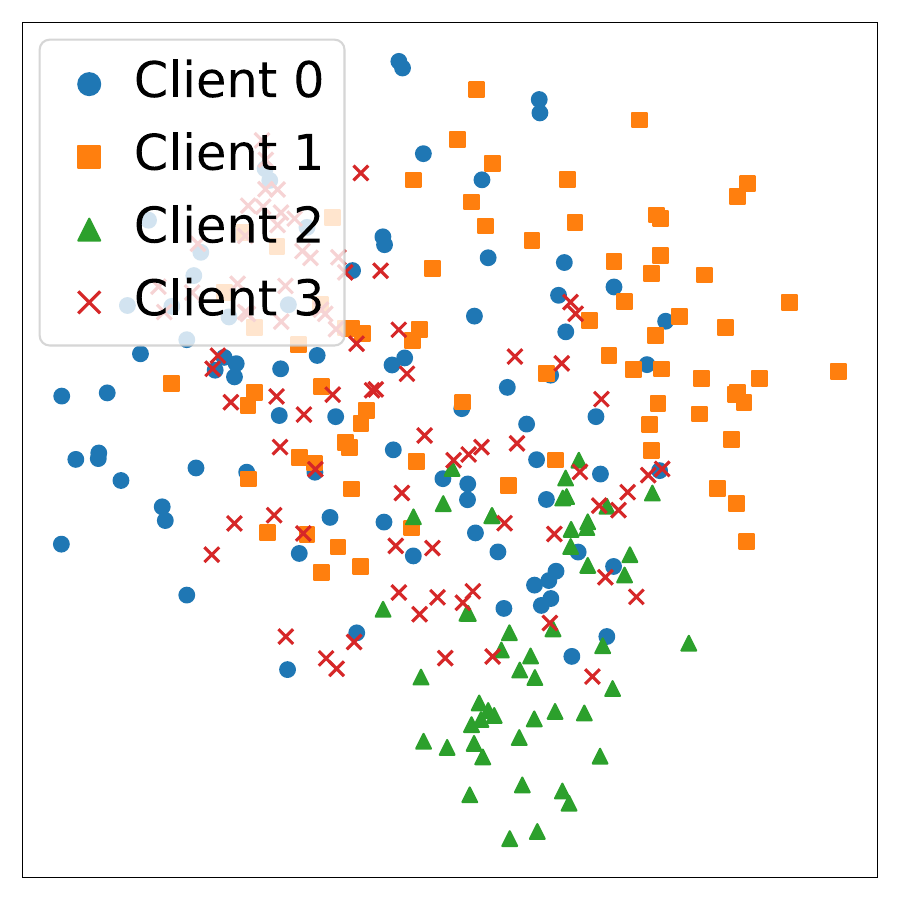}
    \label{subfig-vis-3B-dolly0.5-fusion}
  }
  \subfigure[Features by the last Transformer layer with \modelllama in \datadolly ($\alpha$=5.0).]{
    \includegraphics[width=0.45\linewidth]{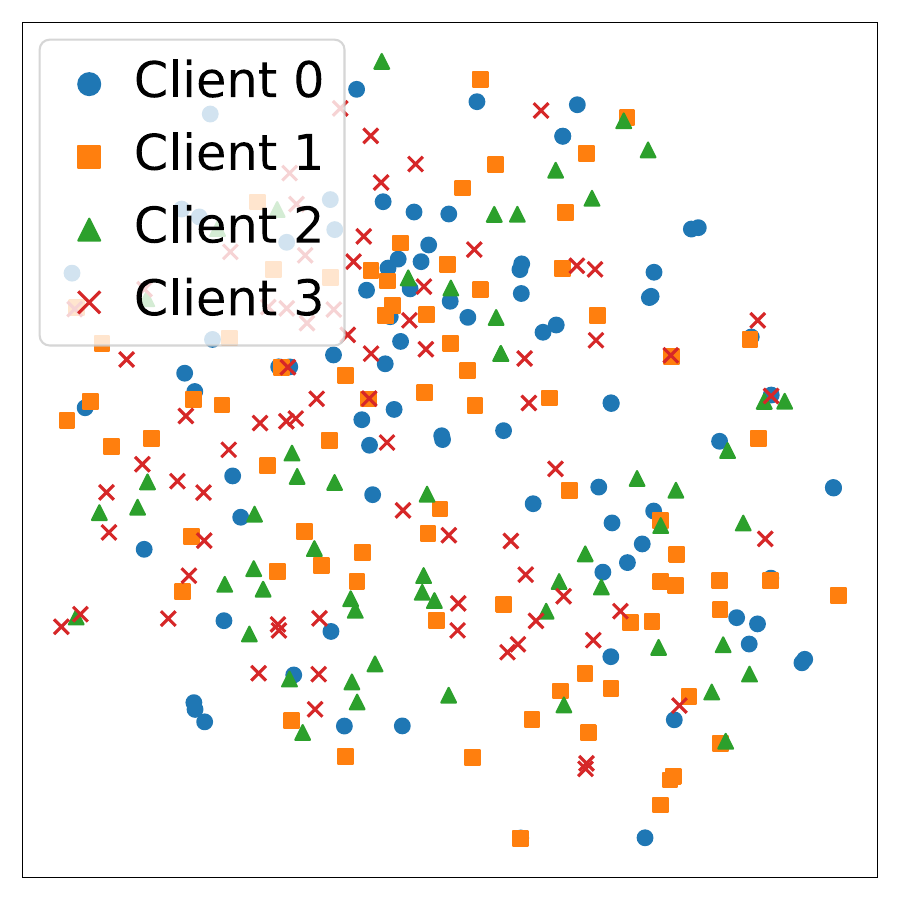}
    \label{subfig-vis-3B-dolly5.0-last}
  }
  \hspace{0.1cm}
  \subfigure[Fused features by \app with \modelllama in \datadolly ($\alpha$=5.0).]{
    \includegraphics[width=0.45\linewidth]{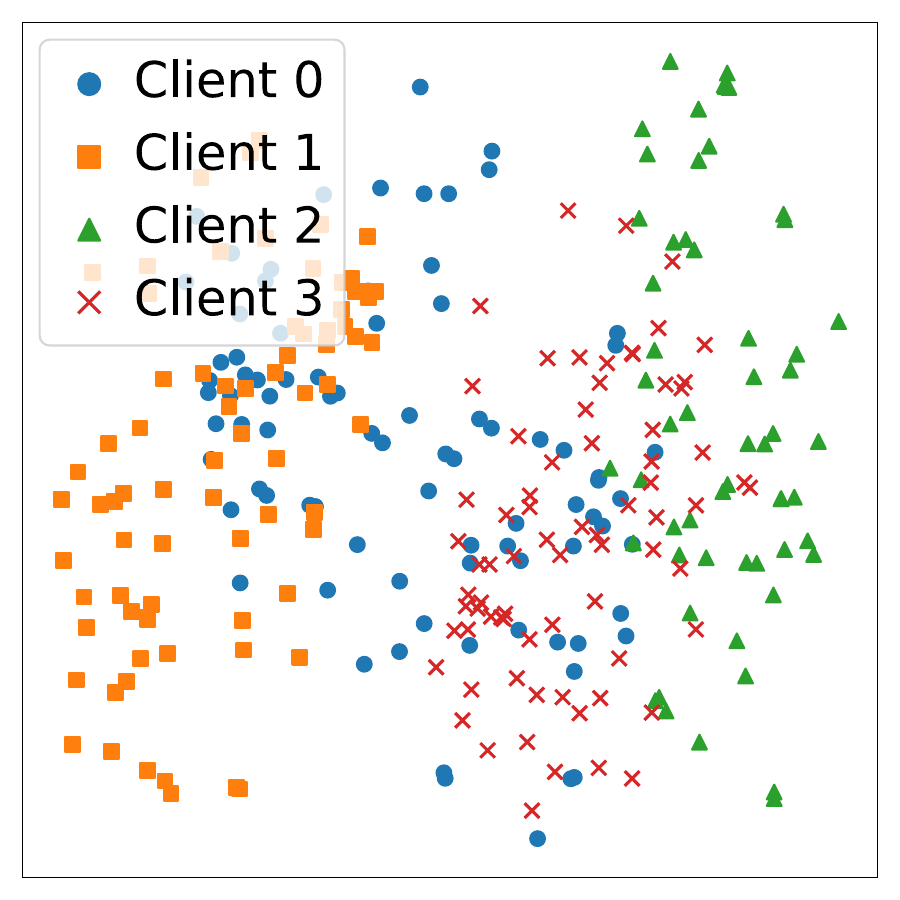}
    \label{subfig-vis-3B-dolly5.0-fusion}
  }
  \subfigure[Features by the last Transformer layer with \modeldatajuicer in \datadolly ($\alpha$=0.5).]{
    \includegraphics[width=0.45\linewidth]{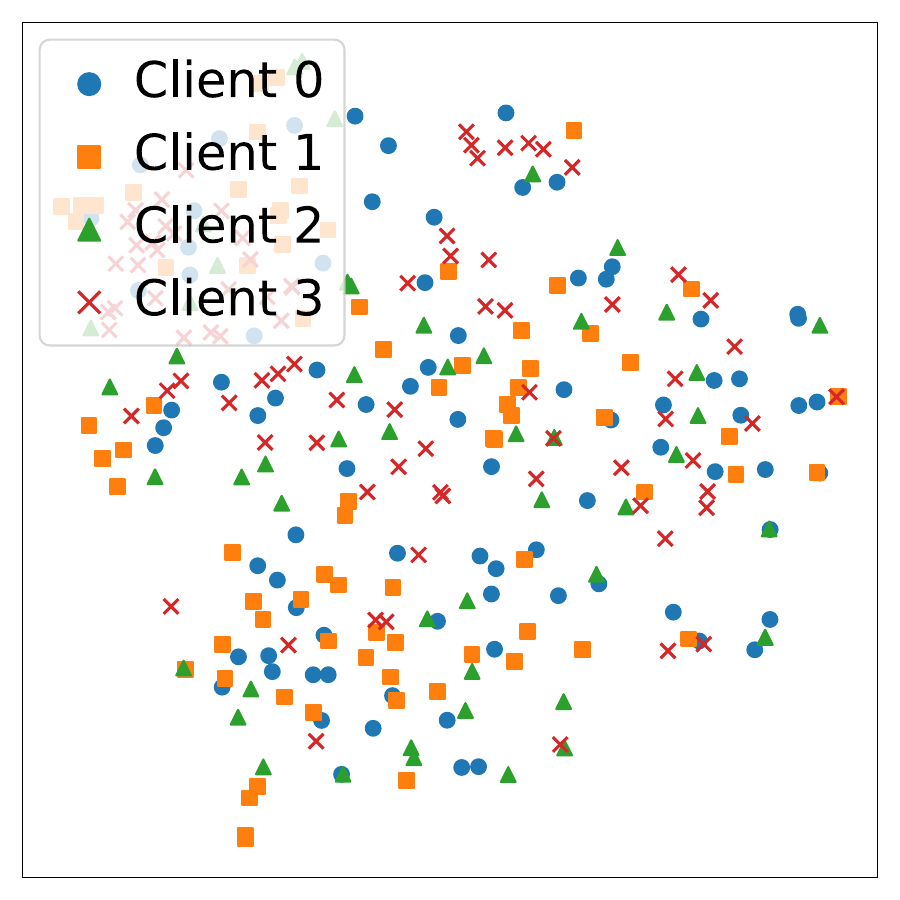}
    \label{subfig-vis-1B-dolly0.5-last}
  }
  \hspace{0.1cm}
  \subfigure[Fused features by \app with \modeldatajuicer in \datadolly ($\alpha$=0.5).]{
    \includegraphics[width=0.45\linewidth]{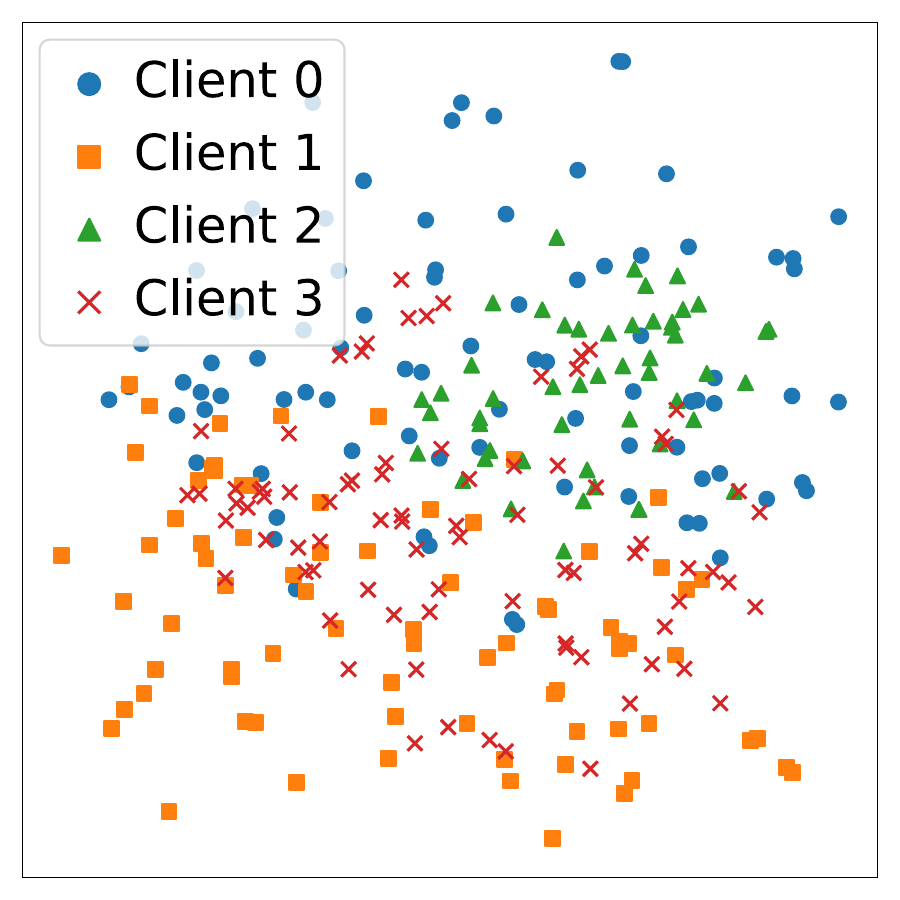}
    \label{subfig-vis-1B-dolly0.5-fusion}
  }
  \caption{Visualization of features from the last Transformer layer and fused features by \app.}
  \label{pic-appendix-visualization-feature-aop}
\end{figure}

\subsection{Performance with Differential Privacy}
\label{subsec:appendix:exp-privacy}
\begin{figure}[t]
\centering
\includegraphics[width=0.6\linewidth]{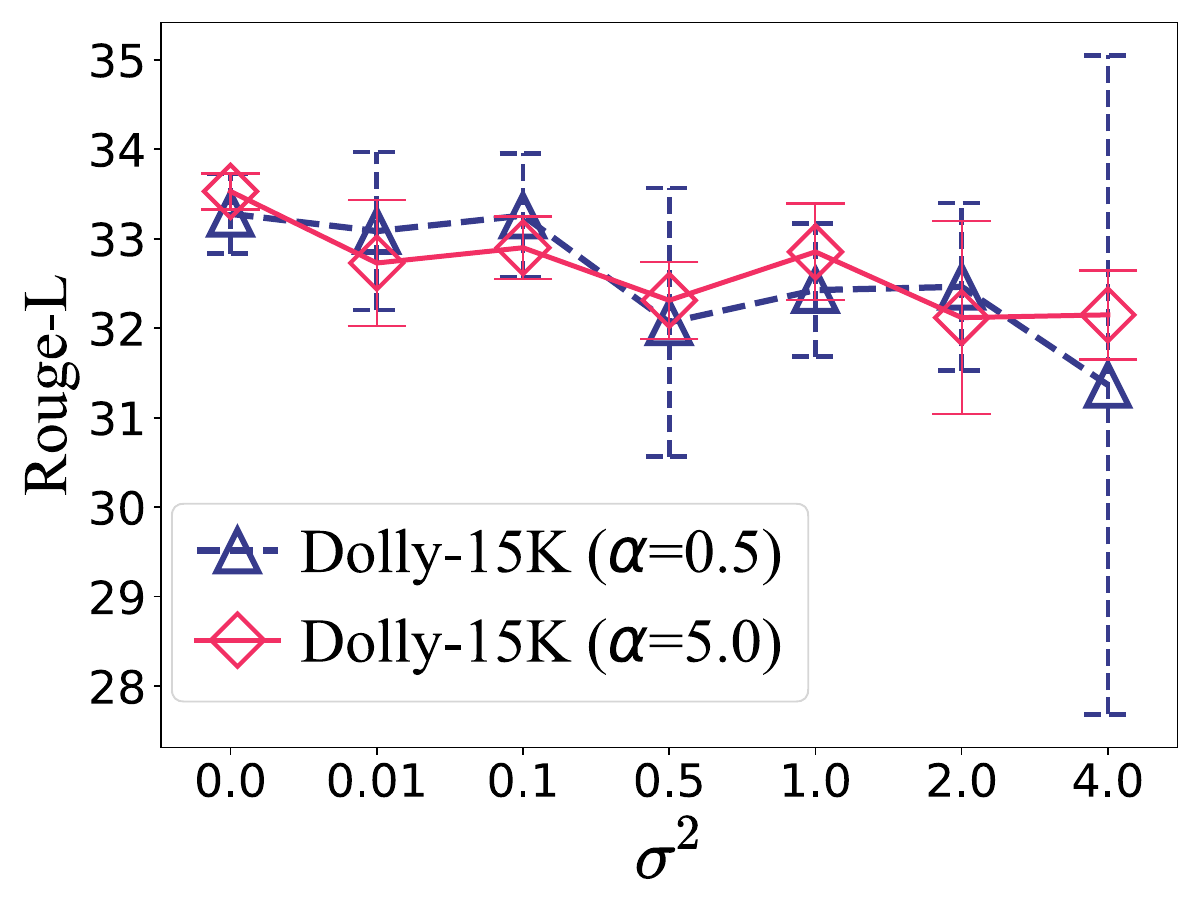}
\caption{Effects of adding DP noise to \apppro (\modeldatajuicer).}
\label{pic-exp-dp}
\end{figure}
Despite that
1) the feature dimensions in our approach are significantly lower compared to the original token-level hidden states, which often have thousands of dimensions, and 
2) the shared centroids do not correspond to real data samples, 
\app still provides additional information compared to the standard FL paradigm \cite{mcmahan2017fl}, such as the spatial distribution of client-side centroids.
As discussed in Section \ref{subsubsec-app-enhancement-privacy}, scaling the elements of the shared centroids to the range [-1,1] and then adding Gaussian noise can enhance the privacy protection of the proposed approaches.

To clarify the impact of adding Gaussian noise, we present example experiments by \apppro on \datadolly with \modeldatajuicer, adjusting the power of the noise by varying its variance.
As presented in Figure \ref{pic-exp-dp}, when the noise variance is no greater than 0.1, \apppro still outperforms Random.
With the noise level increasing further, the performance of \apppro gradually degrades to a level comparable to that of Random as reported in Table \ref{tab-performance}.

\subsection{Breakdown of Time Consumption}
\label{subsec:appendix:exp-time-breakdown}
\begin{table*}[t]
\renewcommand\arraystretch{0.88}
\caption{
Breakdown of time consumption for all steps involving with \app and \apppro \textbf{on \datani}. By default, this table shows the time spent per FL round for each client, with the total time across all clients and rounds in parentheses.
}
\label{tab-time-breakdown-NI}
\centering
\begin{adjustbox}{max width=0.88\textwidth}
\begin{tabular}{l|l|l}
\toprule[1.0pt]
& \multicolumn{1}{c|}{\textbf{\modeldatajuicer on NI}} & \multicolumn{1}{c}{\textbf{\modelllama on NI}} \\
\midrule[1.0pt]
\app       & 
\begin{tabular}[t]{@{}l@{}}
All Steps: 27.5S (11H17M) \\
Training: 0.2S (3M55S) \\
Feature Extraction: 24.8S (10H11M) \\
Feature Fusion: 2.5S (1H1M) \\
Intra-Client Clustering: $6.7 \times 10^{-3}$S (9.9S) \\
Inter-Client Clustering: $4.3 \times 10^{-3}$S (0.2S)
\end{tabular}
&
\begin{tabular}[t]{@{}l@{}}
All Steps: 40.7S (16H42M)\\ 
Training: 6.8S (2H47M)\\ 
Feature Extraction: 30.9S (12H42M)\\ 
Feature Fusion: 3.0S (1H13M)\\ 
Intra-Client Clustering: $6.6 \times 10^{-3}$ S (9.8S)\\ 
Inter-Client Clustering: $3.6 \times 10^{-3}$ S (0.1S)
\end{tabular}
\\
\midrule[0.5pt]
\apppro    &
\begin{tabular}[t]{@{}l@{}}
All Steps: 10.0S (4H7M)\\ 
Training: 1.4S (34M12S)\\
Feature Extraction: 6.6S (2H43M)\\ 
Feature Fusion: 2.0S (49M26S)\\ 
Intra-Client Clustering: $6.7 \times 10^{-3}$ S (9.9S)\\ 
Inter-Client Clustering: $4.3 \times 10^{-3}$ S (0.2S)
\end{tabular}
&
\begin{tabular}[t]{@{}l@{}}
All Steps: 20.1S (8H15M)\\ 
Training: 11.4S (4H40M)\\ 
Feature Extraction: 6.7S (2H46M)\\ 
Feature Fusion: 2.0S (48M20S)\\ 
Intra-Client Clustering: $6.6 \times 10^{-3}$ S (9.7S)\\ 
Inter-Client Clustering: $4.6 \times 10^{-3}$ S (0.2S)
\end{tabular}

\\
\bottomrule[1.0pt]
\end{tabular}
\end{adjustbox}
\end{table*}

\begin{table*}[t]
\caption{
Breakdown of time consumption for all steps involving with \app and \apppro \textbf{on \datadolly}. By default, this table shows the time spent per FL round for each client, with the total time across all clients and rounds in parentheses.
}
\label{tab-time-breakdown-dolly}
\centering
\begin{adjustbox}{max width=0.88\textwidth}
\begin{tabular}{l|l|l}
\toprule[1.0pt]
& \multicolumn{1}{c|}{\textbf{\modeldatajuicer on \datadolly}} & \multicolumn{1}{c}{\textbf{\modelllama on \datadolly}} \\
\midrule[1.0pt]
\app       & 
\begin{tabular}[t]{@{}l@{}}
All Steps: 5.2S (34M52S)\\ 
Training: 2.0S (13M23S)\\ 
Feature Extraction: 3.0S (19M42S)\\ 
Feature Fusion: 0.3S (1M46S)\\ 
Intra-Client Clustering: $1.6 \times 10^{-3}$ S (0.6S)\\ 
Inter-Client Clustering: $8.5 \times 10^{-4}$ S ($3.4 \times 10^{-2}$ S)
\end{tabular}
&
\begin{tabular}[t]{@{}l@{}}
All Steps: 14.6S (2H25M)\\ 
Training: 10.2S (1H41M)\\ 
Feature Extraction: 4.1S (40M37S)\\ 
Feature Fusion: 0.3S (3M21S)\\ 
Intra-Client Clustering: $1.7 \times 10^{-3}$ S (1.0S)\\ 
Inter-Client Clustering: $1.0 \times 10^{-3}$ S ($6.2 \times 10^{-2}$ S)
\end{tabular}
\\
\midrule[0.5pt]
\apppro    & 
\begin{tabular}[t]{@{}l@{}}
All Steps: 4.5S (29M42S)\\ 
Training: 3.6S (23M45S)\\ 
Feature Extraction: 0.7S (4M40S)\\ 
Feature Fusion: 0.2S (1M15S)\\ 
Intra-Client Clustering: $1.6 \times 10^{-3}$ S (0.6S)\\ 
Inter-Client Clustering: $1.0 \times 10^{-3}$ S ($4.0 \times 10^{-2}$ S)
\end{tabular}
&
\begin{tabular}[t]{@{}l@{}}
All Steps: 17.2S (2H52M)\\ 
Training: 16.3S (2H43M)\\ 
Feature Extraction: 0.7S (6M54S)\\ 
Feature Fusion: 0.2S (1M49S)\\ 
Intra-Client Clustering: $1.8 \times 10^{-3}$ S (1.1S)\\ 
Inter-Client Clustering: $1.7 \times 10^{-3}$ S (0.1S)
\end{tabular}
\\
\bottomrule[1.0pt]
\end{tabular}
\end{adjustbox}
\end{table*}
To illustrate the time efficiency of our approach, we provide a detailed breakdown of the time taken by each step, recorded on \datani (Table \ref{tab-time-breakdown-NI}) and \datadolly (Table \ref{tab-time-breakdown-dolly}), respectively.
From these results, we have the following observations:
\begin{itemize}
    \item For \app, the time consumption is usually dominated by feature extraction, due to the relatively high time cost of performing inference with an LLM on all local data samples. 
    \item For \apppro, due to its faster feature extraction than FedHDS, its time consumption is mainly caused by training, feature extraction, and feature fusion. 
\end{itemize}

\subsection{Performance in Various FL Scenarios}
\label{subsec:appendix:exp-performance-scenrios}
\begin{figure}[h]
\centering
\includegraphics[width=0.6\linewidth]{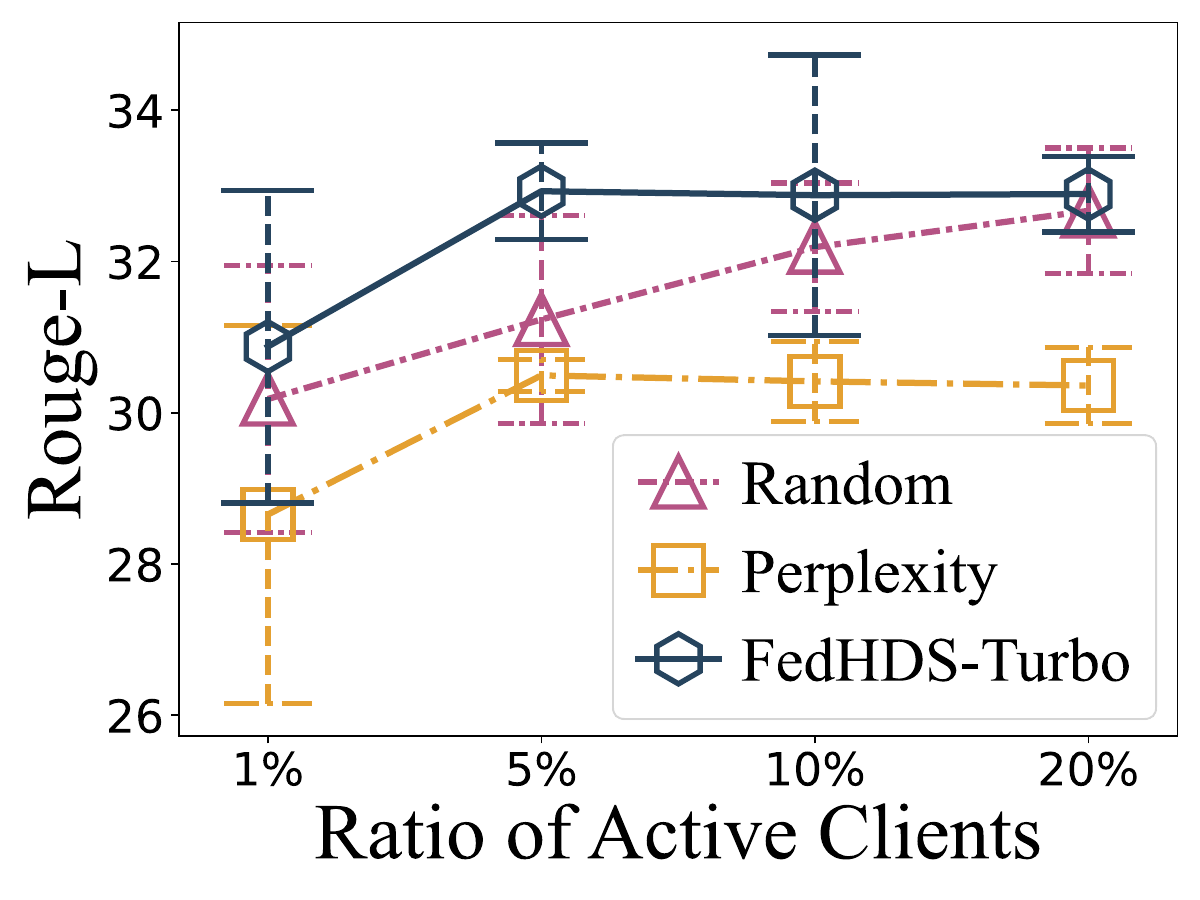}
\caption{Rouge-L with different active ratio on NI with \modelllama.}
\label{pic-exp-active-ratio}
\end{figure}
The ratio of active clients in each round of FL affects the number of centroids sent to the server, potentially affecting the effectiveness of inter-client selection.
We test approaches with coresets in varying proportions of active clients in each round.
As shown in Figure \ref{pic-exp-active-ratio}, \apppro consistently outperforms Random and Perplexity across varying active client ratios. 
When the active client ratio is low (1\%), all approaches perform unsatisfactorily, where the effectiveness of intra-client selection in \apppro may suffer from insufficient centroids. 
With an increasing active client ratio, Random demonstrates a robust growth trend, although it still lags behind \apppro. 
This improvement is likely due to when there are more active clients in each round, the randomly sampled data instances follow a distribution more aligned to the global data distribution.
Considering that clients typically participate in FL with a relatively low active ratio in each round in cross-device FL scenarios \cite{mcmahan2017fl,qin2024full,xu2024forward}, \apppro is more suitable than \app for cross-device settings in terms of both accuracy and efficiency.
\section{Detailed Calculation of Communication Overhead}
\label{sec:appendix:calculation-communication}

Apart from the transmission of model parameters as other federated instruction tuning methods, \app additionally transmits data features of cluster centers and the indices of the selected clusters for data selection. 
Assuming for a client, there are $\rho$ clusters after intra-client selection. The client sends the fused features of the $\rho$ data samples closest to these cluster centers, as: 
\begin{equation}
  [[e_{1,1}, e_{1,2}], [e_{2,1}, e_{2,2}], \ldots, [e_{\rho,1}, e_{\rho,2}]],
\end{equation}
where $e$ denotes a floating number, and $[]$ denotes an array. 
Then, after the inter-client selection, the server returns indices of selected clusters to corresponding clients, which are just a few integers:
\begin{equation}
  [\text{ClusterID}_1, \text{ClusterID}_2, \ldots].
\end{equation}
Therefore, compared to the widely recognized baseline method FedIT, the additional communication overhead of our approach is negligible.

\end{document}